\documentclass[runningheads]{llncs}
\usepackage[T1]{fontenc}
\usepackage{graphicx}
\usepackage{amssymb}
\usepackage{amsmath}
\usepackage{booktabs}
\usepackage{graphicx}
\usepackage{xcolor}
\usepackage{wrapfig}
\usepackage{todonotes}
\newcommand{\qedexample}{\hfill{$\Box $}}

\usepackage{macros}
\usepackage{comment}
\usepackage{enumerate}

\newenvironment{proof}{\par\noindent\textbf{Proof:}\ }{\hfill$\square$\par}

\spnewtheorem*{claimproof}{Proof of Claim}{\itshape}{\normalfont}

\newcommand{\alert}[1]{{\color{red!80}#1}}

\begin{document}
	\title{Rethinking Explanations: Formalizing Contrast in Description Logics}
	% Foundations of Contrastive Explanations in Description logics
	% Explaining Differences, Not Just Facts: A Contrastive View in Description logics
	% From Justifications to Contrast: Rethinking Explanations in Description logics
	% Formalizing Contrastive Explanations in Description logics
	% From Justifications to Context: Formalizing Contrastive Explanations in Description logics
	% Rethinking Explanations: Formalizing Contrast in Description logics
	
	
	\author{Yasir Mahmood \orcidID{0000-0002-5651-5391} 
		\and Arnab Sharma \orcidID{0009-0007-8515-5253} 
		\and
		Axel-Cyrille Ngonga Ngomo 
		\orcidID{0000-0001-7112-3516}
		\and
		Balram Tiwari
	}
	\authorrunning{Y. Mahmood, A. Sharma, A-C. Ngonga, B. Tiwari}
	\institute{Data Science Group, Heinz Nixdorf Institute, Paderborn University Germany
		\email{\{ymahmood, asharma, ngonga, balram\}@uni-paderborn.de}}
	\maketitle              % typeset the header of the contribution

	\begin{abstract}%
		
		There has been a growing interest in explaining entailments over description logic (DL) knowledge bases. 
		The existing explanation formalisms focus on justifications to explain true axioms, and abductive reasoning to explain missing axioms in a knowledge base.
		However, these formalisms only point out the reasoning steps behind a (missing) entailment and  lack a user-centered approach as they do not consider an inquirer's needs, level of understanding, or prior knowledge.
		We propose contrastive explanations, aiming at answering ``why an axiom $P$ (fact) is true instead of another axiom $Q$ (foil)'' over description logic knowledge bases.
		The motivation arises from the observation that when a user discovers that $P$ has occurred, they are often surprised because they anticipated the occurrence of another \emph{similar}  event $Q$.
		Furthermore, individual explanations for ``why $P$'' and ``why not $Q$'' are unsatisfactory since a user expects to see the difference between $P$ and $Q$.
		%   We base our study on Lipton's~\cite{lipton1990contrastive} difference condition and focus on the scenarios when facts and foils are ABox assertions.
		In this work, we first present formal foundations of contrasting questions and then define contrastive explanations within description logics.
		%and then address principles desirable for specific instantiations of this scheme to satisfy.
		To this end, facts include ABox assertions of the form $C(x)$ for a concept $C$ and individual $x$. 
		Possible foils for such facts are assertions $C(y)$ (contrasting against an individual $y$), or $D(x)$ (contrasting against a concept $D$).
		%We argue furthermore that these two choices are natural in the setting of DL KBs. 
		%Our final contribution includes the introduction of various similarity measures to filter the \emph{best} contrastive explanations among all possible ones. 
		Additionally, we explore the properties of contrastive explanations in the DL $\EL$ and $\ALC$. 
		We also provide an implementation of our definition and an experimental evaluation on KBs of varying sizes.
		% which is mostly discussed in the Appendix due to space constraints.
		%Our main contributions include 
		%(1) a general contrastive explanations scheme which can be instantiated by various set-based operations to account for the similarities and differences between them, (2) complexity analysis for the problem to compute an explanation and hardness results
		%(3) several properties enjoyed by various fact-foil combinations.
		
		\keywords{Description logic \and Explanations \and Abduction \and Justifications}
	\end{abstract}

	%Intro: 3 pages
	%Main part: 4 pages
	%	prelim: 1.5 page
	%	new defs: 2.5 pages
	%	Translate Lipton: 
	%	define Lipton and Kean CEs: move example to appendix.
	%	shorten other
	%Preference + EL + Conclusion: 3 pages
	%
	\section{Introduction}

	Description logics (DLs)~\cite{baader2003description,hitzler2009foundations} are foundational to the Web Ontology Language (OWL), enabling precise knowledge compilation and automated reasoning. % essential for semantic web applications. %, especially within ontology engineering and semantic web applications. 
	As a formal framework underpinning OWL, DLs allow for knowledge representation through concepts (classes), roles (properties), and individuals (instances).
	With their well-defined semantics, expressive power, and favorable computational properties, DLs have emerged as fundamental tools in various fields, including ontology engineering~\cite{keet2018introduction}, knowledge representation~\cite{brachman2004knowledge}, and semantic web technologies~\cite{horrocks2003shiq}.
	The success of DLs is corroborated by the development of an increasing number of ontologies based on these formalisms. 
	%\an{Unnecessary sentence ahead} In DLs, the knowledge of a domain is represented via a set of axioms expressing the relationships between the terminology being modelled.
	Within DLs, research has increasingly focused on explaining entailments, where explanations aim to uncover the reasoning that leads to certain conclusions in a knowledge base (KB). 
	In particular, this includes a specific type of explanation known as justifications~\cite{baader1995embedding,schlobach2003non,horridge2011justification,kalyanpur2006debugging}. 
	Given a KB  $\calK$ and an axiom $P$, a justification for $P$ in $\calK$ is a minimal (under set inclusion) subset $\calJ$ of $\calK$ such that $\calJ$ entails $P$. Essentially, a justification answers the question ``why is $P$ true'' by assuming that $P$ is indeed true. However, if $P$ is not true, one natural question arises: what changes should be carried out in $\calK$ to achieve the entailment of $P$ from $\calK$? Abductive reasoning~\cite{peirce1878deduction} offers an approach to address 
	% elsenbroich2006case,wei2014abduction
	this problem. Existing research to this end discusses ABox abduction~\cite{PukancovaH18,Del-PintoS19,Koopmann21a,Koopmann21b}, TBox abduction~\cite{DuWM17,HallandBK14} and concept abduction~\cite{GlimmKW22}. 
	%Motivated by causal reasoning, these approaches seek to capture the elements necessary to generate a desired entailment.
	
	%+++++Old text by Yasir
	%Given a knowledge base  $\calK$ and an axiom $\alpha$, a justification for $\alpha$ in $\calK$ is a minimal (under set inclusion) subset $\calJ$ of $\calK$ such that $\calJ\models \alpha$. \as{Do we want to introduce $\alpha$ here or keeping P?}
	%Essentially, a justification answers the question ``why is $\alpha$ true in $\calK$'' and one assumes that $\alpha$ is indeed true in $\calK$.
	%However, if $\alpha$ is not true, one natural question arises what changes should be employed in $\calK$ to achieve the entailment of $\alpha$ from $\calK$?
	%Abductive reasoning~\cite{peirce1878deduction,elsenbroich2006case,wei2014abduction} offers an approach to address the question of the form ``why is $\alpha$ not true in $\calK$''. 
	%The aim of abduction~\cite{elsenbroich2006case,wei2014abduction} is to explain why an axiom $\alpha$ is not entailed from $\calK$.
	%Formally, a set of axioms $\calH$ is an explanation for $\alpha$ if $\calK\cup\calH$ is consistent and $\calK\cup\calH\models \alpha$. 
	%Existing research on abduction and Description logics discusses ABox abduction~\cite{PukancovaH17,PukancovaH18,Del-PintoS19,Koopmann21a,Koopmann21b}, TBox abduction~\cite{DuWM17,HallandBK14} and concept abduction~\cite{GlimmKW22}. 
	
	%Description logicss (DL) are knowledge representation formalisms specially targeted to encode the termino\text{Math}al knowledge of an application domain. 

	%\paragraph{Contrastive Explanations}
	A justification for $P$ answers ``why P'' by including the actual axioms required to infer $P$ from the given KB.
	The adequacy and relevance of an explanation are determined by the reasoner's inference skills, which do not necessarily consider the user's preferences or interests.
	It has been observed~\cite{lipton1990contrastive} that a ``why'' question is  usually accompanied by a {\em contrast} reflecting the context in which the explanation for the question is demanded. 
	However, the question ``why $P$'' and an answer to it %are %not user-centric, as they 
	do not sufficiently address the aspect of reasoning that interests the user. 
	%When an inquirer questions the explainer (often a reasoner) why a given fact is true, 
	%In other words, the response by the explainer is in connection to the posed query as the explanation consists of the actual steps required to infer the query from the given KB.
	%In particular, the inquirer's specific needs and their level of understanding are often overlooked.
	%In particular, as observed in causality research, causal questions are inherently comparative, typically asking not only ``why $P$'' but also ``why $P$ instead of $Q$''. 
	Integrating contrasts that compare a fact to an expected alternative (or foil), can capture the distinctions that help users understand why a particular outcome prevailed over another.
	In response to these concerns, finding explanations by asking contrastive questions~\cite{lipton1990contrastive,kean1998characterization,Miller_2021,miller2019explanation} has become a popular research area.  %within explainable AI~\cite{stepin2021survey, IgnatievNA020,eiter2023contrastive}. 
	This popularity can be credited to the intuitive nature of such explanations~\cite{miller2019explanation}.
	%
	%The existing explanation services, in particular, 
	From a cognitive perspective, contrastive questions encourage comparison, help decision-making, and improve understanding of the underlying factors~\cite{eiter2023contrastive}. 
	We illustrate this in the following example.
	%~\cite{miller2019explanation}. Consider 
	%Further, a contrastive question allows to include partial information about the inquirer by explicitly stating the contrasting case, revealing which parts of the causal history are clear to them~\cite{eiter2023contrastive}.
	%Further, this approach supports strong mental models, making it easier to adjust to similar situations.
	%Consider an example, when someone asks why you went to the mountains for the holiday, they are looking for a response to the question ``why did you go to the mountains instead of (for example) the seaside''.
	%In this scenario, the cognitive perspective emphasizes the focus on comparative analysis. %When asked why you chose the mountains for your holiday, 
	%More specifically, the inquiry is not about the general notion of taking a break but rather why you opted for that specific destination over other alternatives. %The interest lies in understanding the cognitive processes behind your decision-making, shedding light on factors such as personal preferences, previous experiences, or the benefits of mountains compared to other alternatives.
	%
	\begin{example}\label{ex:Hired-intro}
		Assume the context of hiring at a company, where a model is utilized to classify potential candidates for the position. 
		Suppose an applicant (Alice) was hired for the position but another candidate (Bob) wasn't.
		Then, Bob would be interested in querying ``why Alice was hired and not him'' rather than the mere reason ``why Alice was hired'' or ``why he was not hired''.
		I.e., Bob is interested in knowing the \emph{difference} between him and the successful candidate. \qedexample
	\end{example}
	%The essence of replacing a why question with a contrastive one helps in understanding the reason by not asking what makes a candidate successful, but instead, what differentiates them from such an applicant.
	%We will expound on this example further during the rest of our paper.
	
	%\an{Demonstrates is a strong word. Elucidates?} 
	This example highlights how explicitly stating the contrast allows the inquirer’s explanatory preference to be integrated into the inference process. 
	This viewpoint on contrastive explanations (CEs) offers a clearer understanding of the reasoning step in the presence of alternatives.
	Moreover, the process of highlighting the differences between alternative and equally viable choices promotes more informed decisions. 
	The intuition lies in understanding the processes behind decision-making, shedding light on factors such as preferences, previous experiences, or the advantages of a choice over other alternatives.
%	CEs also facilitate interactive dialogues~\cite{eiter2023contrastive} where users can pose multiple queries to the explainer based on their previous responses.
	%\begin{example}\label{ex:Hired-intro2}
	%   Reconsider the scenario from Example~\ref{ex:Hired-intro}.
	%   Assume a candidate is invited for the interview if he/she is a qualified AI expert.
	%   The system classifies Alice and Bob as qualified; further, Alice is categorized as an AI expert, whereas Bob as a theorist.
	%   After realizing this, Bob seeks clarification on why he was classified as a theorist rather than an AI expert.
	%    An AI expert (resp., a theorist) is a person who specializes in some AI-related (theoretical) topic.
	%    Publishing papers on a topic makes one specialized in that topic. 
	%    Moreover, consider the following facts about two individuals, Alice and Bob. 
	%    
	%    Assume a situation in which an applicant (Alice) was invited for the interview but another candidate (Bob) wasn't.
	%   Then, Bob would be interested in querying ``why Alice was invited instead of him'' rather than the mere reason to ``why Alice was invited for the interview''.
	%    Put differently, Bob is interested in exploring the \emph{difference} between him and the selected candidate. 
	%\end{example}
	
	Contrastive explanations are extensively studied for machine learning (ML) models, especially in the context of explainable AI~\cite{IgnatievNA020,dhurandhar2018explanations,artelt2021efficient,stepin2021survey,IgnatievNA020}. %, and decision support systems. 
	By framing explanations in a contrastive manner, models can provide insights into why a particular decision was made and why it was favored over alternative choices. This approach not only enhances transparency and interpretability but enables users to better understand the underlying factors driving the model's predictions or recommendations. 
	While CEs are commonly utilized to explain the predictions of ML models, they largely remain unexplored in the context of DLs~\cite{Koopmann26}.
	%Why explanation in connection to OWL is needed
	Given that OWL relies on DLs to establish hierarchical relationships and enable complex reasoning within ontologies, this gap is especially pronounced. 
	Notable medical ontologies (such as SNOMED CT~\cite{article} and GALEN~\cite{rector1997}), use DLs to classify entities, supporting inferences about potential disease implications. This foundational reliance on DLs highlights the need for  contrastive explanations to clarify why certain facts prevail over others in knowledge-based reasoning.
	
	To bridge this gap, we propose CEs for DL KBs in this study. 
	%We %do not aim to criticize existing approaches; instead, 
	Our focus is on introducing an additional explanation-based mechanism that is natural and beneficial in various applications where the user's understanding of the decision is crucial. 
	%of the form ``why P instead of Q?'' in the context of DL KBs. 
	%In particular, we are interested in answering why an assertion follows from a knowledge base, whereas another \emph{similar} one does not, by highlighting the causal difference between P and Q.
	To this aim, we introduce contrasting questions
	for an event $P$ which has manifested (known as the \emph{fact}) against an event $Q$ that does not manifest (called the \emph{foil}). 
	This contrasting situation corresponds to the scenario when an inquirer finds out that $P$ has occurred but they expected $Q$ instead.
	We argue that the separate answers to ``why $P$'' given by its justification, and ``why not $Q$'' via its abductive hypothesis, fail to depict the actual difference between $P$ and $Q$. 
	Indeed, in such scenarios, an inquirer receives two explanations that they are to compare themselves.  
	Moreover, there is no way of defining the \emph{best} contrastive explanation since the individual explanations for $P$ and $Q$ may not share \emph{common features}. 
	In this work, we formalize how explanations for $P$ and $Q$ can share common features and identify what the preferred contrastive explanation means under a specific criterion. 
	To summarize, our contributions are as follows.
	\begin{itemize}
		\item We define CEs for DL KBs. To achieve this, we adapt Lipton's~(\cite{lipton1990contrastive}) approach to CEs where facts and foils are formulated as ABox assertions.
		\item We present a general scheme for defining CEs, and give some instantiations for this scheme. We also highlight interesting aspects of each instantiation and their relevance for DL KBs. % and list some properties they satisfy.
		\item We address several criteria for preferred explanations and explore their complexity theoretic properties in two DLs, namely $\EL$ and $\ALC$.% in some detail.
		\item We present a prototypical implementation and evaluate it on several KBs.
	\end{itemize}
	
	%
	%The paper is organized as follows. 
	We next discuss the related work.
	Then, Section~\ref{sec:preliminaries} introduces preliminary notation and Section~\ref{sec:contrast-explanations} describes our formalization of contrastive explanations together with their different aspects for DL KBs.
	Section~\ref{sec:theory} presents an overview of certain theoretical and computational properties for CEs in $\EL$ and $\ALC$, which is followed by an experimental evaluation in Section~\ref{sec:exp}. 
	Finally, Section~\ref{sec:conclusion} concludes our paper  with directions for future work.
	The source code and details on the experimental setup to reproduce our results are available online~\footnote{\url{https://github.com/dice-group/CEs_DLs}}.
	%Throughout this paper, we adopt the term “contrastive explanations (CEs)” to frame our causal approach. Our primary objective is to highlight the causal relationships between facts and foils most relevant to users in DLs reasoning.
	
	\subsubsection{Related Works.}\label{sec:related-work}
	%Currently, justifications, axiom pinpointing, and abduction are the main explanation approaches in DLs.
	\emph{Justification}-based approaches \cite{Kalyanpur2007Finding,Horridge11,KalyanpurPG06} aim at finding explanations for why an axiom is entailed in a knowledge base.
	Other techniques for explaining why questions include proofs~\cite{AlrabbaaBBKK20}.
	%These explanations can also be computed via axiom pinpointing \cite{SchlobachC03,BaaderPS07,BaaderPS07a} %BaaderP10which yields a minimal subset of the knowledge base that entails the given axiom.
	\emph{Abductive} reasoning~\cite{elsenbroich2006case,wei2014abduction} computes a hypothesis that, together with the knowledge base, is sufficient to entail a given set of axioms.
	ABox abduction~\cite{klarman2011abox,PukancovaH18,Del-PintoS19,Koopmann21a} expresses the hypothesis and the query in terms of ABox axioms.
	There also exists research on explanations for consistent query answering, explaining why some answers are (not)  returned under a given repairing semantics~\cite{bienvenu2019computing,lukasiewicz2022explanations}.
	
	%Once abductive explanations or justifications have been obtained, ontology \emph{repairing}~\cite{BaaderKNP18,BaaderKKN21,BaaderKKN22} 
	%provides means to make minimal changes to a given knowledge base, enforcing or preventing a particular entailment. 
	%Moreover, approaches that employ ABox repairing~\cite{BienvenuB16} address the task of making an inconsistent ABox consistent before answering queries on the knowledge base.
	
	Lipton~(\cite{lipton1990contrastive}) defined CEs by citing the causal difference between a fact and the foil. % via the absence of a corresponding event in the history of the foil's failure.
	We base our study on a slight modification of Lipton's model and allow including the part of KB essential in explaining the foil instead of only assertions missing in the history of the foil.
	% Our decision is based on the fact that in some combinations of facts and foils, the known axioms about the foil are also relevant in understanding how missing axioms cause the foil (see Ex.~\ref{ex:running-contrastH}).
	% Lipton's definition does not extend classical justification-based approaches, whereas our approach does. He seems unsure whether ``every apparently non-contrastive question should be analysed in contrastive form''.
	%
	Kean~(\cite{kean1998characterization}) presented a model of abduction using contrasts, where a contrastive explanation cites the difference between the minimal models of the fact and the foil. 
	%Consequently, his CEs are model-theoretic, while our work is based solely on the entailment relation in DLs.
	Recently, CEs have been defined for Answer Set Programs~\cite{eiter2023contrastive} and for propositional logic~\cite{GeibingerJKLV25}. 
	The difference to our work is that we target entailment questions in DLs.
	Finally, \cite{Koopmann26} introduced contrastive explanations for ABox entailment to explain the difference between an individual that is entailed to be in a class and another that is not.
	Although this setting is closely related to our notion of \emph{entity foils}, the underlying definitions differ substantially.
	%In particular, our approach explains a foil within the context of a justification for a specific fact, whereas the aforementioned work seeks to minimize differences between fact and foil at a global level, which may incur significant computational cost.
	%Moreover, the two settings are incomparable and the work in \cite{Koopmann26} is specific in the sense that it can not be extended to our setting.
	Moreover, no prior work targets the notion of contrastive problems involving \emph{concept contrast}.
	
	%
	%In recent years, a number of works have been proposed to compute explanations of an ML model~\cite{IgnatievNM19, IgnatievNA020, IgnatievIS022, xai, HuangIICA022, GCIN20, IzzaIM22} using constraints-solving mechanism. %To this end, these works start by generating a formal representation of the underlying ML model in propositional \text{Math}. Then a formal notion of the explanations is used to compute a (subset/cardinality) minimal explanation. To this end, Ignatiev et al.~\cite{IgnatievNM19} first proposed the concept of {\em abductive} explanations rendering the notion of abductive reasoning from the domain of \text{Math} theory. The notion of abductive explanations therein is defined as the (subset/cardinality) minimal set of {\em feature-value} pairs leading to the output prediction for an input to the ML model. 
	%
	Research on contrastive explanations also exist for reasoning in argumentation frameworks~\cite{BorgB22-1,BorgB22} and classification in machine learning~\cite{IgnatievNA020,dhurandhar2018explanations,artelt2021efficient,stepin2021survey,LuLKD20}.
	However, in the ML setting, abduction answers the ``why'' question, while CEs answer the ``why not'' question. % in ML literature. 
	Further, the notion of CEs from ML cannot be applied directly in our setting, which requires rigorous and formal succinctness.
	%Moreover, they also defined {\em contrastive-explanations} for a specific input-output pair . This is defined as the (subset/cardinality) minimal set of input features that if allowed to take some other values, and when the remaining features do not change their values in the corresponding input, the output prediction is changed.  
	%To compute abductive/contrastive explanations the aforementioned works encoded an ML model in \text{Math}al constraints and then applied a satisfiability modulo theory (SMT) solver to compute explanations of different examples on the model. 
	% Since the ML model described in SMT theory often could take a longer time to compute explanations, in~\cite{IgnatievIS022} Ignatiev et al. proposed an satisfiability (SAT) encoding of the ensemble tree models and used maximum satisfiability (MaxSAT) approach to compute abductive explanations.
	%Finally,~\cite{amgoud2022axiomatic} presented axiomatic foundations of explanations for ML models and introduced a list of axioms desirable for such explainers.
	%
	
	\section{Preliminaries}\label{sec:preliminaries}
	
	% \todo[inline]{It doesn't make sense to introduce ALC as a basic Description logics but then use ALCH in all the examples: either give examples and ALC, or just introduce ALCH as basic DL. \\
		% use a,b for individuals rather than x,y. \\
		% }
	We give a short exposition on the DL $\ALC$~\cite{baader2003description,hitzler2009foundations}.
	The choice is for simplicity rather than a restriction, and our setting is not confined to specific DLs.
	%The main ingredients in DLs include individuals, concepts (unary relations), and roles (binary relations) to provide relationships between individuals.
	%As a result, DLs are fragments of first-order predicate \text{Math} using only unary and binary predicates.
	
	Let $N_I$, $N_C$, and $N_R$ denote mutually disjoint sets of individual, concept,
	and role names, respectively. 
	In $\ALC$, concepts are built by the following grammar rules $C \ddfn A \mid  \neg C \mid  C \sqcap C \mid C \sqcup C \mid  \exists r.C \mid \forall r.C$ where $A \in N_C$ and $r \in N_R$.
	We also use two special symbols $\top$ and $\bot$ to represent \textit{everything} and \textit{nothing}, respectively.
	%We use abbreviations when writing complex concepts, thus $\bot\dfn A\sqcap \neg A$, $\top \dfn \neg \bot$, $C\sqcup {D}\dfn \neg(\neg C\sqcap \neg {D})$, and $\forall r.C\dfn \neg(\exists r. \neg C)$.
	A TBox is a finite set of general concept inclusions (GCIs), i.e., axioms of the form $C \subsum  D$ for concepts $C,  D$. %and depicts relations between the concepts (or terms). 
	An ABox assertion is an expression of the form $C(a)$ (concept assertion) or
	$r(a, b)$ (role assertion), for $a, b \in N_I$, concept $C$, and role $r \in N_R$. 
	%An assertion refers to knowledge about individuals and the relationships between them.
%	An ABox is a finite set of assertions. 
	%We often refer to ABox assertions as ABox axioms as well.
	Finally, a KB is a pair $\calK = (\calT,\calA)$ (often also seen as $\calT\cup\calA$) where $\calT$ and $\calA$ are TBox and ABox, respectively. The signature of $\calK$ is the set of all concepts, roles and individual names that appear in $\calK$.
	\begin{table}[t]
		%\vspace{-.5cm}
		%\begin{wraptable}{r}{0.55\textwidth}
		% 'r' for right side, adjust width as needed
		\centering
		\caption{Syntax and semantics for $\mathcal{ALC}$: $\mathcal{I}$ is an interpretation with domain $\Delta^\mathcal{I}$.}
		\label{tab:semantics}
		\setlength{\tabcolsep}{3pt}
		\resizebox{.6\columnwidth}{!}{
			\begin{tabular}{l  c  c }
				\toprule
				Construct           & Syntax         & Semantics \\
				\midrule
				Atomic concept          & $A$            & $A^{\mathcal{I}}\subseteq{\Delta^\mathcal{I}}$\\
				Role                    & $r$            & $r^\mathcal{I}\subseteq{\Delta^\mathcal{I}\times \Delta^\mathcal{I}}$\\
				Top concept         & $\top$         & $\Delta^\mathcal{I}$\\
				Bottom concept      & $\bot$         & $\emptyset$            \\
				Conjunction          & $C\sqcap  D$    & $C^\mathcal{I}\cap  D^\mathcal{I}$\\
				Disjunction             & $C\sqcup  D$    & $C^\mathcal{I}\cup  D^\mathcal{I}$\\
				Negation                & $\neg C$       & $\Delta^\mathcal{I}\setminus C^\mathcal{I}$\\
				Existential restriction & $\exists r.C$ & $\{ x \mid \exists~ y\in C^\calI, (x,y) \in r^\mathcal{I}\}$\\
				Universal restriction & $\forall r.C$   & $\{ x \mid \forall~ y. (x,y) \in r^\mathcal{I} \implies y \in C^\mathcal{I}\}$\\ 		
				\bottomrule
			\end{tabular} 
		}\vspace{-.5cm}
		%\end{wraptable}
	\end{table}
	The semantics of $\ALC$ is defined in terms of interpretations. 
	An interpretation $\calI$ is a tuple $\calI = (\Delta^\calI, \cdot^\calI)$, where $\Delta^\calI$ is a non-empty set called the domain of $\calI$, and $\cdot^\calI$ is a function that maps every individual name $a \in N_I$ (resp., concept $C\in N_C$, and role $r\in N_R$) to an element $a^\calI \in \Delta^\calI$ (subset $C^\calI \subseteq \Delta^\calI$ and binary relation $r^\calI \subseteq  \Delta^\calI\times \Delta^\calI$).
	The interpretation function can be extended to arbitrary concepts as usual (see Table~\ref{tab:semantics}).
	
	Let $C\subsum  D$ be a GCI and $\calI$ be an interpretation.
	Then, $\calI$ satisfies $C\subsum  D$, denoted as $\calI\models C\subsum  D$ iff $C^\calI\subseteq  D^\calI$.
	Similarly, $\calI$ satisfies an assertion $C(a)$ iff $a^\calI\in {C}^\calI$; and $r(a, b)$ iff $(a^\calI , b^\calI ) \in r^\calI$.
	We write axiom to mean either a TBox axiom or an ABox assertion.
	$\calI$ is a model of the KB $\calK$ ($\calI \models \calK$), iff $\calI$ satisfies every axiom in $\calK$.
	Finally, %let $\calK$ be a DL KB and $\alpha$ be an axiom, then 
	$\calK\models \alpha$ iff $\calI\models\alpha$ for every model $\calI$ of $\calK$.
	%The knowledge base in the following example will serve as a running example for our discussion.
	\begin{example}\label{ex:running-prelim}
		Reconsider the scenario from Example~\ref{ex:Hired-intro}.
		Our KB models that candidates for the position must be qualified AI experts. 
		An AI expert (resp., a theorist) is someone wo publishes in some AI (theoretical) domain.
		Further, working in the AI domain also makes a person AI expert.
		The ABox contains assertions about two individuals. 
		Alice and Bob are both qualified; Alice primarily publishes in ML, an AI topic, whereas Bob publishes in pure \text{Math}, considered as a theoretical topic. 
		%	We let, $\calKrun =(\calT, \calA)$ where $\calT$ and $\calA$ contain the following axioms:  
		We let, $\calKrun =(\calT, \calA)$ where $\calT$ contains the following axioms:
		\begin{description}
			\item[(T1)] $\concept{Qualified}\sqcap \concept{AI\text-Expert} \subsum \concept{Hired}$, 
			\item[(T2)] $ \exists\concept{publishesIn}.\concept{AI}\subsum \concept{AI\text-Expert}$,
			\item[(T3)]  $\exists\concept{publishesIn}.\concept{Theory}\subsum\concept{Theorist}$, \textbf{(T4)} $ \exists\concept{worksIn}.\concept{AI}\subsum \concept{AI\text-Expert}$.
		\end{description}
		The ABox includes the following assertions.
		\begin{description}
			\item[(A1)] $\concept{Qualified} (\text{Alice})$,  \textbf{(A2)} $\concept{Qualified} (\text{Bob})$,
			\item[(A3)] $\concept{Publishes} (\text{Alice}, \text{ML})$,  \textbf{(A4)} $\concept{AI}(ML)$, 
			\item[(A5)] $\concept{Publishes} (\text{Bob}, \text{Math})$,  \textbf{(A6)} $\concept{Theory}(\text{Math})$. 
			% \textbf{(A7)} $\neg \concept{AI}(\text{Math})$. \qedexample
		\end{description}
	\end{example}
	
	\subsubsection*{Reasoning Problems in DLs.}
	Let $\calK$ be a KB and $\alpha$ be an axiom, such that  $\calK\models \alpha$.
	Then, we are interested in the subset-minimal part of $\calK$ (called a \emph{justification} for $\alpha$) relevant for this entailment.
	By a slight abuse of the notation, a justification is a subset of a KB, as $\calJ\subseteq \calT\cup \calA$ for a KB $\calK=(\calT,\calA)$.
	\vspace{-.1cm}
	\begin{definition}[Justification]\label{def:justification}
		Let $\calK=(\calT,\calA)$ be a KB, %with TBox $\calT$, ABox $\calA$, 
		and let $P\dfn C(x)$ be a query (an ABox assertion). Then, a \emph{justification} for $P$ in $\calK$ is a subset  $\calJ\subseteq \calK$ such that (1) $\calJ\models P$, and (2) $\calJ'\not\models P$ for every $\calJ'\subset \calJ$.
	\end{definition}
	We denote by $\JustAll(P,\calK)$ the set of all justifications for $P$ in $\calK$.
	Also, it is convenient to denote a justification for $P$ by $\calJ(P)$ when the context allows. 
	\begin{example}\label{ex:running-just}
		Consider the KB $\calKrun$ (Ex.~\ref{ex:running-prelim}) and let $P\dfn \concept{Hired}(Alice)$.
		A justification for $P$ is given by: $\concept{Qualified}\sqcap \concept{AI\text-Expert} \subsum \concept{Hired}$,  $ \exists\concept{publishesIn}.\concept{AI}\subsum \concept{AI\text-Expert}$,
		$\concept{Qualified} (\text{Alice})$, $\concept{Publishes} (\text{Alice}, \text{ML})$, $\concept{AI}(ML)$.
		%	Therefore, we have $\calJ(P)\dfn \{\textbf{T1},\textbf{T2}, \textbf{A1}, \textbf{A3}, \textbf{A4}\}$.
	\end{example}    
	For non-entailment, one finds hypotheses that retrieve the entailment. %, achieved via
	%. ~\cite{klarman2011abox,PukancovaH18}.
	\begin{definition}[Abductive hypothesis]\label{def:abduction}
		Let $\calK=(\calT,\calA)$ be a KB, % with TBox $\calT$, ABox $\calA$, 
		and $Q\dfn C(x)$ be a query (ABox assertion). An \emph{abductive hypothesis} for $Q$ in $\calK$ is a set $\calH$ of ABox axioms over the signature of $\calK$ such that (1) $\calK \cup \calH$ is consistent, (2) $\calK \cup \calH\models Q$, and (3) $\calH$ is subset-minimal for which (1) and (2) hold.
	\end{definition}
	
	%Notice that $\calK\models Q$ implies that $\calH=\emptyset$ is a trivial abductive hypothesis for $Q$ in $\calK$.
	%As a result, we are mainly interested in cases where $\calK\not\models Q$.
	We denote by $\AbdAll(Q,\calK)$ the set of all abductive hypotheses for $Q$ in $\calK$.
	Further, let $\calH\in \AbdAll(Q,\calK)$, an \emph{abductive justification} (or simply, a justification) for $Q$ in $\calK$ is a justification for $Q$ in $ \calK\cup \calH$. % where $\calH$ is an abductive hypothesis for $Q$ in $\calK$.
	%That is, we define justifications for $Q$ in $\calK$ using an abductive hypothesis, although $\calK\not\models Q$.
	We denote the set of all justifications for $Q$ in $\calK$ by %(by slightly abusing the notation) 
	$\JustAll(Q,\calK)=\{\calJ \mid \exists \calH\in\AbdAll(Q,\calK), \calJ\in \JustAll(Q, \calK\cup\calH)\}$.

	%We additionally call an abduction problem \emph{signature-based}, if our hypothesis set $H$ is only allowed use concepts and roles in a given signature $\Sigma$. 
	
	\begin{example}\label{ex:running-abd}
		Consider $\calKrun$ (Ex.~\ref{ex:running-prelim}) and let $Q\dfn \concept{Hired}(\text{Bob})$.
		Clearly, $\calKrun\not\models Q$, and we find the following abductive hypotheses for $Q$: $\{\concept{Hired}(\text{Bob})\}$ (which is trivial),
		$\calH_0= \{\concept{AI\text-Expert}(Bob)\}$,
		$\calH_1=\{\concept{publishesIn} (\text{Bob}, \text{ML})\}$, and
		$\calH_2=\{\concept{worksIn} (\text{Bob}, \text{ML})\}$.
		%	Note that, the hypotheses $\calH_1$ and $\calH_2$ can be generalized in to a family of abductive hypotheses, namely $\calH^j_1=\{\concept{Specializes} (\text{Bob}, x_j), \concept{AI}(x_j)\}$  and $\calH^j_2=\{\concept{Publishes} (\text{Bob}, x_j), \concept{AI}(x_j)\}$ for $j\in\mathbb N$.
		To be concise, we will consider $\calH_1$ and $\calH_2$ in the following discussions.
		The justifications $\calJ_i(Q)$ for $Q$ corresponding to $\calH_i$ ($i\in \{1,2\}$) include the following axioms, respectively.  
		\begin{enumerate}
			\item
			$\concept{Qualified} (\text{Bob})$,   \alert{$\concept{publishesIn} (\text{Bob}, \text{ML})$}, $\concept{AI}(ML)$, \\ $ \exists\concept{publishesIn}.\concept{AI}\subsum \concept{AI\text-Expert}$, $\concept{Qualified}\sqcap \concept{AI\text-Expert} \subsum \concept{Hired}$.
			\item
			$\concept{Qualified} (\text{Bob})$,   \alert{$\concept{worksIn} (\text{Bob}, \text{ML})$}, $\concept{AI}(ML)$, \\
			$\exists\concept{worksIn}.\concept{AI}\subsum \concept{AI\text-Expert}$,
			$\concept{Qualified}\sqcap \concept{AI\text-Expert} \subsum \concept{Hired}$. \qedexample  
			%		\item
			%		$\concept{Qualified} (\text{Bob})$,   \underline{$\concept{Publishes} (\text{Bob}, \text{x})$}, $\underline{\concept{AI}(x)}$.
			%\item %[$\calJ_2$]$\concept{Qualified} (\text{Bob})$, $\concept{Publishes} (\text{Bob}, \text{Complexity})$, \underline{\concept{AI}(Complexity)}.
		\end{enumerate}
	\end{example}
	\begin{comment}
		\begin{example}[Continue.]\label{ex:prelim-abd}
			Reconsider the KB from Example~\ref{ex:prelim-just} and let $Q\dfn C(y)$.
			Clearly $\calK\not\models Q$.
			Then, two abductive hypotheses for $Q$ are $\calH_1= \{B_2(y)\}$ and $\calH_2 =\{A_2(y)\}$ since $\calH_i\cup \calK\models Q$ for $i=1,2$.
			Further, $\calH_1\cup \{B_1(y), B_1\sqcap B_2 \subsum C \}$ and  $\calH_2\cup \{A_1(y), A_1\sqcap A_2 \subsum C \}$ constitute the corresponding justifications for $Q$ in $\calK$.
		\end{example}
	\end{comment}
	%

	\section{Contrastive Explanations in Description Logics}\label{sec:contrast-explanations} 
	Recall that a \emph{fact} is a true axiom which logically follows from a KB, and a \emph{foil} is a false statement that is not entailed.
	In this work, we focus on explaining ABox assertions.
	As a result, facts ($P$) and foils ($Q$) both take the form $C(x)$ or $ D(y)$ for concepts $C, D$ and individuals $x,y$. For simplicity, we focus on concept assertions in this work, although one can also consider role assertions.

	Let $\calK= (\calT,\calA)$ be a KB, and % which we will call an explanation context.
	$P, Q$ be two assertions over the signature of $\calK$.
	%We call an axiom $\alpha$ a \emph{fact} if $\calK\models \alpha$, and a \emph{foil} otherwise. 
	Then, we call the pair $(P, Q)$ a \emph{contrastive question} (CQ) if $P$ is a fact and $Q$ is a foil.
	Observe that one can in general contrast a fact $P$ against any foil $Q$. 
	%That is, as long as $\calK\models P$ and $\calK\not\models Q$, the pair $(P,Q)$ is a valid CQ.
	Nevertheless, we adhere to Kean's~(\cite{kean1998characterization}) observation that a contrast for a fact is \emph{of the same type} as the fact itself.
	Moreover, Lipton~(\cite{lipton1990contrastive}) also motivated certain combinations of facts and foils, having a \emph{largely similar history, against which the differences stand out}.
	These observations lead to exploring the combinations of facts and foils that are interesting and meaningful for DL KBs.
	%
	\begin{comment}
		\begin{example}\label{ex:running-fact-foil}
			In the explanation context $\calKrun$, $\concept{Hired}(\text{Alice})$ is a fact and $\concept{Hired}(\text{Bob})$ is a foil. % since Alice is Hired and Bob is not. 
			Further, $\concept{Theorist}(\text{Bob})$ is another fact with the foil $\concept{AI\text-Expert}(\text{Bob})$ since Bob is asserted to be a theorist instead of an AI expert.
		\end{example}
	\end{comment}
	%
	%\ym{R1: this section uses fact and foil informally and then defines based on these a CQ; and Sec. 3.1 where it is a bit difficult to grasp what is intended without having details on what a contrastive explanation is}
	%
	\subsubsection{Instantiation for Facts and Foils}
	We instantiate $P$ and $Q$ in the contrastive question $(P,Q)$ with DLs concept assertions.
	Precisely, we let $P\dfn C(x)$ be a fact for a concept $C$ and individual $x$.
	Following the previous discussion, %Kean~\cite{kean1998characterization}, 
	%sit appears that 
	the three most intriguing foils for contrasting $C(x)$ include $\neg C(x), C(y)$,  and $D(x)$. % yields a reasonable foil for $C(x)$.
	The approach to contrasting a fact against its negation involves computing justifications and has been explored in-depth~\cite{kalyanpur2006debugging,horridge2011justification}.
	Therefore, the classical justifications for ABox assertions can be considered a special case of our problem to computing CEs. %as  $ \AbdAll(Q)=\emptyset$. %as $\calK$ is consistent. %and therefore there is no abductive explanation for $Q$.
	However, the remaining two types of foils, i.e., $C(y)$ and $D(x)$, have not gained (enough) attention before.
	Consequently, we focus on the cases when a foil is obtained from a fact $C(x)$ by diverging the individual element $x$ (\emph{entity} contrasts) and when a concept is changed, resulting in $D(x)$ (\emph{concept} contrasts).
	These instantiations arise naturally for the contrast as the two cases share certain \emph{common features} with the given fact.
	The notion of entity contrasts bears some resemblance to the existing work on explaining the difference between two individuals~\cite{Koopmann26} within a KB.
	Nevertheless, the underlying definitions differ substantially as the aforementioned work only considers ABox assertions to explain the difference.
	Moreover, our approach aims to explain a foil within the context of a justification for the given fact, whereas the work by \cite{Koopmann26} seeks to minimize the differences between fact and foil in terms of the ABox assertions for both, which may incur significant computational cost.
	%Finally, no prior work targets the notion of contrastive problems involving \emph{concept contrast}.

	\paragraph{Entity Contrasts.}
	Corresponding to facts of the form $C(x)$, we let foils $C(y)$, where $x,y$ are two distinct individuals and $C$ is a concept.
	Choosing a contrasting individual $y$ to compare the assertions $C(y)$ and $C(x)$ highlights the differences in relevant characteristics encompassed within the concept $C$. 
	That is, the CQ $(C(x), C(y))$ allows one to contrast an individual inside $C$ with the one not in $C$.
	This emphasizes specific features of $C$, aiding in understanding why certain entities are asserted to be positive instances in $C$. 
	%The following example depicts the relevance of entity contrasts.
	%
	\begin{example}\label{ex:entiy-contrast}
		In $\calKrun$ (Ex.~\ref{ex:running-prelim}),
		$\concept{Hired}(\text{Alice})$ is a fact.
		For a skeptical user interested in finding the differences between Alice and Bob regarding the hiring process, $\concept{Hired}(\text{Bob})$ serves as an entity foil. \qedexample
	\end{example}
	\paragraph{Concept Contrasts.}
	We propose a foil $D(x)$ for a fact of the form $C(x)$, where $C$ and $D$ are two distinct concepts and $x$ is an individual.
	Introducing a contrasting concept for the same individual enables exploring the boundary between different concepts within a KB. 
	%That is, contrasting $C(x)$ with the assertion $D(x)$ clarifies situations where $x$ is asserted to be in $C$ but not in $D$ in the given KB.
	%This allows a deeper understanding of the subtle relationships between two concepts.
	%
	%Example~\ref{ex:concept-prelim} depicts a scenario emphasizing  the relevance of entity contrasts.
	
	\begin{example}\label{ex:concept-prelim}
		In $\calKrun$ (Ex.~\ref{ex:running-prelim}), %upon finding out that Bob was not classified as an AI expert, 
		% in a slightly different setting.
		%Upon receiving the explanation for why Bob wasn't invited for the interview, one realizes that he is categorized as a theorist (say, the system automatically categorized people based on the conference types they mostly publish in).
		%Now, 
		one may want to know further why Bob is classified as a theorist instead of an AI expert.
		%    Our KB $\calK$ is specified as before.
		Then, corresponding to the fact ``$\concept{Theorist}(\text{Bob})$'', the assertion ``$\concept{AI\text-Expert}(\text{Bob})$'' serves as   a concept foil. \qedexample
	\end{example}
	%
	%We illustrate a potential use-case of finding an explanation for an assertion against its concept contrast. 
	%In the area of class expression learning~\cite{}, the goal is to construct a query from positive and negative examples. Then, a skeptic user may inquire why the answer was a  query $C$ instead of another similar and expected query $D$? 
	%
	\subsection{Formalizing Contrastive Explanations}
	%We begin by presenting an example to provide an intuitive understanding of what constitutes a contrastive explanation according to Lipton~(\cite{lipton1990contrastive}).
	%\begin{example}\label{ex:running-cex-intuition}
	%	Consider a CQ $(\concept{Hired}(\text{Alice}),\concept{Hired}(\text{Bob}))$ in the explanation context $\calKrun$.
	%
	%	As a contrastive explanation for why Alice is Hired instead of Bob, one expects the difference that Alice publishes in ML, an AI topic. Thus, Alice is an AI expert, while Bob is not. \qedexample
	%\end{example}
	
	%Example~\ref{ex:running-cex-intuition} captures our intuitive understanding of the difference between Alice and Bob in terms of known assertions about both.
	
	%\todo[inline]{	\textbf{Structure:} \textbf{Issue 1} (P1) discussion, trivial solution (symdif), what is lacking? \textbf{Issue 2} and \textbf{Issue 3} same. \textbf{Our definition} for CEs. Move additional examples to appendix. \\}
	
	We build upon the framework of Lipton~(\cite{lipton1990contrastive}) and first present an account on what constitutes a CE.
	Essentially, contrasting $P$ and $Q$ amounts to citing a difference between the reasons for $P$ and the missing facts to entail $Q$ in a KB $\calK$.
	This so-called \emph{difference condition} \cite{lipton1990contrastive,Miller_2021} requires identifying a \textbf{cause} for $P$ and the \textbf{absence} of \textbf{corresponding events} in the history of $Q$'s \textbf{failure}.
	%We aim to define CEs in KBs by providing meaningful translations of each highlighted term.
	%The principle of difference condition % first appeared in the original work by %Kean~\cite{kean} and 
	%Lipton~\cite{lipton}. It 
	%requires CEs to include axioms that differentiate $P$ and $Q$ in a KB $\calK$.
	%Observe that Lipton did not give an account on how contrastive explanation can be computed because his interest was only in modeling such explanations. Essentially, Lipton only highlighted the characteristics of formalism to qualify for contrastive explanations. Kean then implemented some of these ideas by Lipton and defined CEs based on model-theoretic notions. We adapt Lipton's principles to DL KBs and define CEs. 
	
	\begin{property}[Difference condition~\cite{lipton1990contrastive}]\label{prop:diff}
		A contrastive explanation for $(P, Q)$ cites the difference between the reasons for $P$ and $Q$.
	\end{property}
	%Nevertheless, both authors do not assume specifically whether $Q$ is not true.
	%\ym{can we formulate a precise statement out of it? a condition, a formula that one can check whether is true or not. E.g., given R1 and R2 as reasons for P and Q: maybe their set difference is enough as it cites the difference. so the pair (diff(P,Q), diff(Q,P)). Then can be motivated from Lipton, Kean and Miller as they provide an instantiation for it }
	
	%We present a scheme for defining CEs (Def.~\ref{def:CeXp}), where the principle of difference condition is achieved by combining the justifications for $P$ and $Q$.
	%Precisely, we compute a justification %$\calJ(P)$ for $P$ in $\calK$.
	The difference condition (P1) allows one to formulate a contrastive explanation scheme by combining the justifications for $P$ and $Q$.
	Since $Q$ is not true in $\calK$, a justification $\calJ(Q)$ is determined by an abductive hypothesis $\calH(Q)$ in $\calK$.
	%Let $\calK$ be a KB, $(P,Q)$ be a CQ, and  $\calJ(P)$ and $\calJ(Q)$ be justifications for $P$ and $Q$, respectively.
	%Then, a contrastive explanation for $(P,Q)$ is given by \emph{appropriately} joining justifications for $P$ and $Q$ in $\calK$.
	We define a general CE scheme for ``$P$ instead of $Q$'' based on the \emph{difference condition} ``$\difference$'' as a set of axioms $\calE \dfn \calJ(P) \difference \calJ(Q)$.
	%Next, we instantiate the difference condition via specific operations on the two sets.
	%
	\begin{comment}
		\begin{definition}[Contrastive explanations scheme]\label{def:CeXp}
			Let $\calK=(\calT,\calA)$  be an explanation context and $(P,Q)$ be a CQ. 
			A \emph{contrastive explanation} for ``$P$ instead of $Q$'' based on the difference condition ``$\difference$'' in $\calK$ is a set $\calE$ of axioms such that 
			$\calE \dfn \calJ(P) \difference \calJ(Q)$ 
			where $\calJ(P)$ and $\calJ(Q)$ are justifications for $P$ and $Q$ in $\calK$, respectively.
		\end{definition}
	\end{comment}
	%
	Our general definition translates Lipton's difference condition to facts and foils formulated in DLs, except it leaves room for the precise interpretation of the difference condition ``$\difference$''.
	%Let $\calK$ be a KB and $R= (P,Q)$ be a CI.
	%
	%For a difference operation $\difference$, an explanation context $\calK$, and a CQ $R\dfn(P,Q)$, we denote by $\ContAll{\difference}(R,\calK)$ the set of all CEs for $R$ in $\calK$.
	%We consider various instantiations for ``$\difference$'' and their associated challenges. %in the context of DL KBs.
	%
	As our first instantiations, we consider the symmetric difference (defined as $A\symdif B \dfn (A\setminus B) \cup (B\setminus A)$).
	This defines a CE $\calE^\symmcontrast \dfn \calJ(P) \symdif \calJ(Q)$ for $(P,Q)$.
	
	% of ``$\difference$'', .
	\begin{comment}
		\begin{definition}\label{def:symdif}
			Let $\calK=(\calT,\calA)$ be a KB and $(P,Q)$ be a CQ. 
			A \emph{contrastive explanation} for $(P,Q)$ in $\calK$ based on the symmetric difference is the set of axioms
			$\calE^\symmcontrast \dfn \calJ(P) \symdif \calJ(Q)$ 
			where $\calJ(P)$ and $\calJ(Q)$ are some justifications for $P$ and $Q$ respectively.
		\end{definition}
	\end{comment}
	
	%\ym{R2: Definition 3 implements Principle 1 (P1). Then, Definition 4 instantiates Definition 3 with symmetric difference operator. I wonder why the authors didn't introduce Definition 4 directly without Definition 3 because Definition 3 is not used for something else.}
	%Then, the contrastive explanations for $(P,Q)$ in a KB $\calK$ based on the symmetric difference are defined as $\calE^\symmcontrast  \dfn \calJ(P)\symdif \calJ(Q)$ where $\calJ(X)\in \JustAll(X,\calK)$ for  $X\in\{P,Q\}$.
	%
	%$\ContAll{\symmcontrast}  \dfn \{\calJ(P)\symdif \calJ(Q)\mid \calJ(X)\in \JustAll(X,\calK), X\in\{P,Q\}\}.$.
	
	\begin{example}\label{ex:running-symdif}
		Consider the CQ $(\concept{Hired}(\text{Alice})$, $ \concept{Hired}(\text{Bob}))$ in the KB $\calKrun$. % (Ex.~\ref{ex:Hired-entity-prelim}).
		\begin{comment}
			Recall that $\calJ(\concept{Hired}(\text{Alice}))$ contains following axioms:  
			\begin{enumerate}
				\item $\concept{Qualified}\sqcap \concept{AI\text-Expert} \subsum \concept{Hired}$,  
				
				\item $ \exists\concept{publishesIn}.\concept{AI}\subsum \concept{AI\text-Expert}$,
				
				\item $\concept{Qualified} (\text{Alice})$, $\concept{publishesIn} (\text{Alice}, \text{ML})$, $\concept{AI}(ML)$.  
			\end{enumerate}
			%
			For $i\in\{1,2\}$, $\calJ_i(\concept{Hired}(\text{Bob}))$ (Ex.~\ref{ex:running-abd}) specifies two justifications for Bob, as follows.  
			%Each $\calJ_i$ includes axioms ($1$)--($i$) from above and additionally the following assertions, respectively:
			\begin{itemize}
				\item
				$\concept{Qualified} (\text{Bob})$,   {$\concept{publishesIn} (\text{Bob}, \text{ML})$}, $\concept{AI}(ML)$, \\ $ \exists\concept{publishesIn}.\concept{AI}\subsum \concept{AI\text-Expert}$, $\concept{Qualified}\sqcap \concept{AI\text-Expert} \subsum \concept{Hired}$.
				
				\item
				$\concept{Qualified} (\text{Bob})$,   {$\concept{worksIn} (\text{Bob}, \text{ML})$}, $\concept{AI}(ML)$,, \\
				$ \exists\concept{worksIn}.\concept{AI}\subsum \concept{AI\text-Expert}$,
				$\concept{Qualified}\sqcap \concept{AI\text-Expert} \subsum \concept{Hired}$.  
			\end{itemize}
		\end{comment}
		%
		Using the justifications for fact (Ex.~\ref{ex:running-just}) and foil (Ex.~\ref{ex:running-abd}), the CQ $(\concept{Hired}(\text{Alice})$, $ \concept{Hired}(\text{Bob}))$ admits the following CEs (among others).  
		%\an{Does it make sense to underline the element of $\calH$ to make clear that  $\calE_{1,2}^\symmcontrast$ actually makes sense.}
		\begin{itemize}
			\item $\calE_1^\symmcontrast =  \{\concept{Qualified} (\text{X}), \concept{publishesIn} (\text{X}, \text{ML}) \mid \text{X}\in \{\text{Alice}, \text{Bob}\} \}$.
			
			\item $\calE_2^\symmcontrast = \{\concept{Qualified} (\text{Alice}), \concept{Publishes} (\text{Alice}, \text{ML})$,
			\item[] \qquad\qquad $\concept{Qualified} (\text{Bob}), \concept{worksIn} (\text{Bob}, \text{ML}),$
			\item[]\qquad \qquad $\exists\concept{publishesIn}.\concept{AI}\subsum \concept{AI\text-Expert}, \exists\concept{worksIn}.\concept{AI}\subsum \concept{AI\text-Expert}\}$. \qedexample  
		\end{itemize}
	\end{example}
	%
	%$\calE^\symmcontrast$ includes axioms exclusively for either $P$ or $Q$, thus satisfy the difference condition.
	
	Note that \cite{lipton1990contrastive} considers a CE to include the cause of $P$ and the \textbf{absence} of an event in the history of $Q$'s \textbf{failure}.
	Therefore, a direct translation of Lipton's definition to DLs applies the difference condition to the justification $\calJ(P)$ and the abductive hypothesis $\calH(Q)$ for a CQ $(P,Q)$.
	I.e., Lipton's contrastive explanations are defined as sets 
	$\calE^{\hypcontrast}  \dfn \calJ(P)\symdif \calH(Q)$.  %where $\calJ(P)\in \JustAll(P,\calK)$ and $ \calH(Q)\in\AbdAll(Q,\calK)$.
	Since $\calJ(P)\subseteq \calK$ and $\calH(Q)\cap \calK=\emptyset$, we have $\calE^\hypcontrast \dfn \calJ(P)\cup \calH(Q)$.
	In the following, we present a short exposition on how $\calE^{\hypcontrast}$ fails to provide a promising definition  in DLs.
	%
	%We highlight this further via the following example.
	\begin{comment}
		\begin{example}\label{ex:contrast-issue1}
			Consider the KB from Example~\ref{ex:prelim-just} together with the fact $P\dfn C(x)$ and foil $Q\dfn C(y)$.
			Recall that $\calJ(P)= \{A_1(x),A_2(x), A_1\sqcap A_2 \subsum C \}$ is a justification for $P$ in $\calK$.
			Moreover, %$\calJ({Q})= \{B_2(y), B_1(y), B_1\sqcap B_2 \subsum C \}$ %and 
			$ \calJ({Q})= \{A_1(y), A_2(y), A_1\sqcap A_2 \subsum C \}$ 
			is a justification for $Q$ in $\calK$.
			%
			Notice that $\calJ(P)\cap \calJ(Q) =\{A_1\sqcap A_2 \subsum C\}$ and therefore $\calE^\symmcontrast = \{A_i(x), A_i(y) \mid i=1,2\}$.
		\end{example}
	\end{comment}
	%
	\begin{example}\label{ex:running-contrastH}
		The CQ $(\concept{Hired}(\text{Alice})$, $ \concept{Hired}(\text{Bob}))$ admits the following CEs $\calE^\hypcontrast_i$ for $i\in \{1,2\}$ in $\calKrun$.
		Each $\calE^\hypcontrast_i$ contains the TBox axioms listed in \textbf{(T1)}--\textbf{(T2)} from Example~\ref{ex:running-prelim} and the following ABox assertions.  
		\begin{itemize}
			\item $\calE_1^\hypcontrast=  \{ \concept{Qualified} (\text{Alice}), \concept{publishesIn} (\text{Alice}, \text{ML}), \concept{AI}(ML)$,
			\item[] \qquad\qquad  $\concept{publishesIn} (\text{Bob}, \text{ML})\}$.
			
			\item $\calE_2^\hypcontrast= \{ \concept{Qualified} (\text{Alice}), \concept{publishesIn} (\text{Alice}, \text{ML}), \concept{AI}(ML)$, 
			\item[] \qquad \qquad $\concept{worksIn} (\text{Bob}, \text{ML})\}$. \qedexample
		\end{itemize}
	\end{example} 
	
	The explanation $\calE^\hypcontrast$ includes true axioms required to infer $P$ and missing axioms that highlight the failure of $Q$.
	However, $\calE^\hypcontrast$ alone may be insufficient to fully explain how $Q$ can be inferred solely from the missing axioms for the foil.
	The following example depicts such a corner case.
	
	\begin{example}\label{ex:issue-lip}
		%	Let $\calK' \dfn \calKrun\setminus \{\neg \concept{AI}(\text{\text{Math}})\}$ be a KB.
		Let $\calH_c = \{\concept{AI}(\text{\text{Math}})\}$ be our abductive hypothesis for $\concept{Hired}(\text{Bob})$ in $\calKrun$.
		A CE $\calE_c^\hypcontrast$ for the CQ from our running example in $\calKrun$ contains the TBox axioms \textbf{(T1)}--\textbf{(T2)} from Example~\ref{ex:running-prelim} and the following assertions. \vspace{-.1cm}
		\begin{itemize}
			\item $\calE_c^\hypcontrast = \{ \concept{Qualified} (\text{Alice}), \concept{publishesIn} (\text{Alice}, \text{ML}), \concept{AI}(ML), \alert{\concept{AI}(\text{\text{Math}})}$\}. \vspace{-.1cm}
		\end{itemize}
		
		However, the assertion $\concept{AI}(\text{\text{Math}})$ in this CE does not clarify why is $\concept{Hired}(\text{Bob})$ missing in $\calKrun$.
		Informally, one additionally requires (at least) the assertion $\concept{publishesIn}(\text{Bob},\text{\text{Math}})$ to understand how $\concept{Hired}(\text{Bob})$ is entailed by $\concept{AI}(\text{\text{Math}})$ in $\calK'$, which is part of its justification but not included in $\calE_c^\hypcontrast$. \qedexample \vspace{-.1cm}
	\end{example}
	
	%Observe that the difference condition merely requires highlighting the difference between justifications for $P$ and $Q$ in $\calK$ and does not distinguish true axioms in $\calJ(Q)\cap \calK$ from those added by an abductive hypothesis for $Q$.
	%Furthermore, including only the hypothesis yields $\calE^{\hypcontrast}$, which we already discussed above.
	We observe that relying solely on the difference condition may not fully deliver on the promise of explaining the difference between facts and foils.
	%    As the examples above highlight, answering a contrastive question by satisfying P1 may not always be fruitful in DL KBs.
	%    To be precise, the CeXp $\{A_1(x), A_2(x), A_1\sqcap A_2 \subsum C, B_2(y) \}$  for $(P,Q)$ does not inform how $Q$ is deducable from $B_2(y)$.
	This holds since $\calE^\symmcontrast$ satisfies P1, yet it does not distinguish between true and missing assertions for $Q$ ($Q$'s \textbf{failure}), whereas $\calE^\hypcontrast$ only contains an abductive hypothesis for $Q$ instead of all the axioms necessary to entail $Q$.
	%The symmetric difference between a justification for $P$ and abductive hypotheses for $Q$ is effective only in certain scenarios (as depicted in Example~\ref{ex:contrast-issueH}). 
	%    However, it fails to capture the essence of contrastive explanations as a difference.
	%
	%This motivates additional instantiations for the difference condition to account for the difference between \emph{true} versus \emph{hypothesized} assertions 
	%and \emph{corresponding} assertions 
	%about a foil in a CQ.

	\paragraph{DL-Specific Instantiations for CEs.}
	We now propose two properties desirable for an instantiation of the difference condition ``$\difference$'' to fulfill for CQs formulated in DLs.
	These properties model Lipton's intuition of the \emph{similarity} and \emph{difference} between the justifications for $P$ and $Q$.
	In essence, P1 accounts for axioms common in explaining the fact $P$ and foil $Q$, which are irrelevant in explaining the difference between $P$ and $Q$.
	% Nevertheless, using an abductive hypothesis $\calH$ instead of a justification $\calJ_Q$ for Q will give $\calJ_P$ as the difference since $\calH\cap\\calJ =\emptyset$.
	However, the justification $\calJ(Q)$ contains both true and missing assertions about $Q$, which should be distinguished to highlight the \textbf{failure} of $Q$ in $\calK$ (Property~\ref{prop:missing}).
	Moreover, the two justifications $\calJ(P)$ and $\calJ(Q)$ may contain \textbf{similar} axioms, irrelevant to differentiate a fact from a foil (Property~\ref{prop:similar}).
	%We will consider this similarity of axioms in Principle~$3$.
	%
	We explain both principles against their possible instantiations. % for ``$\difference$''. %as depicted in Table~\ref{tab:diffs}.
	
	\begin{property}[Known versus assumed knowledge]\label{prop:missing}
		A CE for $(P,Q)$ differentiates true assertions in the reasoning about $Q$ %in $ \calK$ (i.e., $\calK\setminus\calJ_Q$) 
		from missing assertions for $Q$.% ($\calH_Q$).
	\end{property}

	Principle~\ref{prop:missing} (P2) aims at answering the CQ $(P,Q)$ by elaborating on the missing assertions causing $Q$'s failure. 
	Thus, an instantiation satisfies P2 if true and missing assertions about $Q$ can be differentiated in the given CE.
	It can be seen that the CEs $\calE^\symmcontrast$ does not satisfy P2, whereas, $\calE^\hypcontrast$ satisfies P2 trivially.
	
	%In particular, P2 emphasizes underlining what would make $Q$ true in $\calK$ and corresponds to Lipton's contrastive explanations (jointly with P1) by considering the \emph{missing facts in the history of the failure of Q}.
	%
	\begin{comment}        
		Consider Kean's~\cite{kean1998characterization} definition of contrastive explanation, defined as 
		$\calE^{\hypneg}  \dfn \calJ(P)\symdif \neg\calH(Q) $ where $ \calJ(P)\in \JustAll(P,\calK) $, $\calH(Q)\in\AbdAll(Q,\calK)$ and
		$\neg A$ denotes the negation of all the assertions in $A$. \as{Why A here?}
		%   Then, explanations in $\ContAll\hypneg$ satisfy P2 as the missing assertions about $Q$ are negated in the hypothesis.
		%    However, as we have seen before only hypothesis for $Q$ may fail 
		However, the definition $\ContAll\hypneg$ fails for expressive DLs that allow negation as a language construct.
		Consequently, Kean's approach to contrastive explanations can not be adapted to facts and foils formulated in Description logicss. \ym{not super interesting}
	\end{comment}
	
	We next instantiate $\difference$ based on the symmetric difference with partition ($\calE^{\symmpart\text{-f}}$) by further separating true assertions from assumed assertions about $Q$.
	Let $\calJ(Q)\in \JustAll(Q,\calK)$ and $\calH(Q)$ be the abductive hypothesis for $Q$ in $\calK$ such that $\calH(Q)\subseteq \calJ(Q)$.
	We denote by $\calJ_t(Q)\dfn \calJ(Q)\cap \calK$ the axioms in $\calJ(Q)$ \emph{true} for $Q$.
	Then, define CEs $\calE^{\symmpart\text{-f}}$ by further partitioning the justifications for $Q$.
	
	\begin{definition}\label{def:seg}
		Consider a KB $\calK=(\calT,\calA)$ and a CQ $(P,Q)$. %\an{This is not a sentence.}
		A \emph{contrastive explanation} for $(P,Q)$ in $\calK$ based on the symmetric difference with full partition is a tuple $\calE^{\diffpair\text{-f}}\dfn \langle \calJ(P)\setminus \calJ_t(Q), \calJ_t(Q)\setminus\calJ(P), \calH(Q)\rangle $ 
		where $\calJ(X)\in \JustAll(X,\calK)$ for $X\in\{P,Q\}$ with $\calJ_t(Q)=\calJ(Q)\cap\calK$ and $\calH(Q) =  \calJ(Q)\setminus \calK$.
	\end{definition}
	
	\newcommand{\fact}{\text{fact}}
	\newcommand{\foil}{\text{foil}}
	\newcommand{\truefoil}{\text{foil+}}
	\newcommand{\missingfoil}{\text{foil-}}
	
	$\calE^{\diffpair\text{-f}}$ differentiates axioms for $P$ and $Q$ as well as true assertions ($\calJ_t(Q)$) from the missing assertions ($\calH(Q)$) about $Q$.
	For brevity, we rename each individual component in a CE $\calE$ as $\calE_\fact$ (axioms for the fact), $ \calE_\truefoil$ (true axioms for the foil), and $\calE_\missingfoil$ (missing axioms for the foil), respectively.
	%As a result, the definition $\calE^{\diffpair\text{-f}}$ satisfies P1 and P2.
	%    We find it worth noticing that the definition $\calE^\diffpair$ resembles the contrastive explanation for ASPs by Eiter~et~al.\cite{} since the authors also distinguish true assertions in KB from missing assertions about $Q$ by presenting a tuple as an explanation.
	
	\begin{example}\label{ex:running-diffpair}
		For our running example, two CEs with full partition are \linebreak $\calE_1^{\diffpair\text{-f}}=\langle \calE_\fact,\calE_\truefoil, \calE_\missingfoil\rangle$ and $\calE_2^{\diffpair\text{-f}}=\langle  \calE'_\fact,\calE'_\truefoil, \calE'_\missingfoil\rangle$ where,  
		
		\begin{description}
			\item[$\calE_\fact$] $= \{\concept{Qualified} (\text{Alice}), \concept{publishesIn} (\text{Alice}, \text{ML})$, 
			
			\item[$\calE_\truefoil$] $=\{\concept{Qualified} (\text{Bob})\}$,
			
			\item[$\calE_\missingfoil$] $=\{\concept{publishesIn} (\text{Bob}, \text{ML})\}$, and
			\vspace{.20em}
			
			\item[$\calE'_\fact$]  $= \{\concept{Qual} (\text{Alice}), \concept{publishesIn} (\text{Alice}, \text{ML}), \exists\concept{publishesIn}.\concept{AI}\subsum \concept{AI\text-Expert}\}, $
			
			\item[$\calE'_\truefoil$] $=\{\concept{Qualified} (\text{Bob}), \exists\concept{worksIn}.\concept{AI}\subsum \concept{AI\text-Expert}\}$,
			
			\item[$\calE'_\missingfoil$] $=\{\concept{worksIn} (\text{Bob}, \text{ML})\}$.
		\end{description}
		%Intuitively, the two explanations can be read as follows:
		%\begin{itemize}
		%	\item Alice is Hired because she is qualified and publishes in ML. Publishing in a topic makes one specializes in that topic. Although Bob is also qualified, he does not (according to $\calK$) specialize in ML.
		%	\item Alice is Hired because she is qualified and publishes in ML. Although Bob is also qualified, he does not (according to $\calK$) publish in ML. \qedexample 
		%		\item Alice is Hired because she is qualified and publishes in ML, which is an AI topic. Although Bob is qualified and publishes in complexity, it is not (known as) an AI topic.
		%\end{itemize}
	\end{example} 
	%It is worth mentioning that $\calE^{\diffpair\text{-f}}$ also differentiates axioms about $P$ from $Q$; therefore, it also satisfies P1.
	%The missing assertions (hypothesis) for $Q$ in Example~\ref{ex:running-diffpair} are \emph{negated} in the interpretation. 
	%
	\begin{comment}
		%Kean~(\cite{kean1998characterization}) defines CEs as $\calE^{\hypneg}  \dfn \calJ(P)\symdif \neg\calH(Q)$, where $\neg\calH(Q)$ denotes the negation of axioms in $\calH(Q)$.
		%Since Kean also considers an abductive hypothesis (rather than justifications) for the foil, a similar issue arises as with Lipton's definition, rendering $\calE^{\hypneg}$ insufficient for DLs.
		\begin{remark}
			$\calE^{\diffpair\text{-f}}$ satisfies P1 and P2. \ym{maybe mention this result at the end for all the considered definitions together}
		\end{remark}
	\end{comment}
	
	A final challenge in our instantiation concerns the ``common'' assertions. 
	The two assertions about Alice and Bob being qualified are both required and true in $\calKrun$.
	As a result, these should be treated as \emph{similar} since they do not differentiate the fact from the foil.
	Our final principle allows us to treat \emph{similar} (yet unequal) axioms for $P$ and $Q$. 
	Note that this characteristic %has not been explored previously, although it can also 
	corresponds to Lipton's argument of citing the cause of $P$ with no \textbf{corresponding event} in the history of the failure of $Q$.
	That is, the difference condition removes the same axioms in an explanation and similar ones for a fact-foil pair. 
	%is also desirable for KBs formulated in Description logicss, as we explain in the following.
	
	\begin{property}[Accounting for similar axioms]\label{prop:similar}
		A CE for $(P,Q)$ does not differentiate between ``similar'' true assertions about $P$ and $Q$. % obtainable via renaming. % of the assertions about $Q$.
	\end{property}
	
	%\begin{example}
	%	The assertions $\concept{Qualified} (\text{Alice})$ and $\concept{Qualified} (\text{Bob})$ are similar for the CQ $R= (\concept{Hired} (\text{Alice}), \concept{Hired} (\text{Bob}))$.
	%	Both assertions are true as required, and do not explain the difference between the given fact and foil. %$\concept{Hired} (\text{Alice})$ and $ \concept{Hired} (\text{Bob})$ in $\calK$. 
	%	\qedexample
	%\end{example}
	
	The standard set difference does not account for such similar axioms between $\calJ(P)$ and $\calJ(Q)$.
	Indeed, each CE based on the set difference contains assertions $\concept{Qualified} (\text{Alice})$ \& $\concept{Qualified} (\text{Bob})$.
	%
	%\begin{remark}
	%	None of the CEs $\calE^\symmcontrast, \calE^{\symmpart\text{-p}} , \calE^{\diffpair\text{-f}}$ satisfy P3. Further, the CEs by Lipton ($ \calE^\hypcontrast$) %and Kean ($\calE^\hypneg$) 
	%	satisfy P3 trivially as $\calE^\hypcontrast$ does not include true assertions about $Q$ present in an explanation context $\calK$.
	%\end{remark}
	%
	Moreover, the definition of \emph{similarity} between axioms depends on how a given fact and foil are related.
	We now specify criteria for similarity when entity and concept foils are considered.
	Intuitively, for entity CQs $(C(x), C(y))$, two ABox axioms $E(x)$ and $E(y)$ are considered the same for any concept $E$.
	%That is, if $E(x)$ (resp., $E(y)$) is required to entail $C(x)$ ($C(y)$) and $\{E(x),E(y)\}\subseteq \calK$, then a contrastive explanation treats the two assetions $E(x)$ and $E(y)$ as similar.
	Likewise, for concept CQs $(C(x), D(x))$, two TBox axioms $E\subsum C$ and $E\subsum D$ can be treated similar.
	%We systematically formalize this intuition next. % that obtains an axiom by renaming individuals or concepts.
	
	\begin{definition}[Renamed axioms]\label{def:renamed-axioms}
		Let $\alpha$ be an axiom %(TBox or ABox)
		and $\gamma$ be an individual (resp., a concept name) in $\alpha$.
		Then, for an individual (concept name) $\delta$: $\alpha[{\gamma\mapsto\delta}]$ denotes the axiom obtained from $\alpha$ by renaming every occurrence of $\gamma$ by $\delta$.
	\end{definition}
	
	\begin{example}
		For $\alpha \dfn \concept{Qualified} (\text{Alice})$,  $\alpha[{\text{Alice}\mapsto\text{Bob}}]= \concept{Qualified} (\text{Bob})$.
		For $\beta =\concept{Theorist} (\text{Bob})$, $\beta[{\text{Theorist}\mapsto\text{AI-Expert}}] = \concept{AI\text-Expert} (\text{Bob})$.\qedexample
	\end{example}
	
	To obtain our final instantiation, we modify the set difference operator to obtain $\contrastdiff{X}{Y}$, allowing the removal of \emph{similar} axioms.
	
	\begin{comment}
		\begin{definition}
			Let $R=(P,Q)$ be a CeXp-instance with a fact $P=C(x)$ and $Q\in \{C(y), D(x)\}$ be a foil.
			Moreover, let $\calJ_P$ and $\calJ_Q$ be justifications for $P$ and $Q$ in $\calK$.
			Then, we define by $\contrastdiff$ the modified set-difference between the two justifications for $P$ and $Q$, specified as follows:
			\begin{itemize}
				\item If $Q= C(y)$, then $\calJ_P\contrastdiff \calJ_Q\dfn \{\alpha\in \calJ_P\mid  \alpha\not\in \calJ_Q \text{ and } \alpha \neq E(x) \text{ for concept } E \text{ with } E(y) \in \calJ_Q\}$.
				\item If $Q= D(x)$, then $\calJ_P\contrastdiff \calJ_Q\dfn \{\alpha\in \calJ_P\mid  \alpha\not\in \calJ_Q \text{ and } \alpha \neq E(x) \text{ for concept } E \text{ with } \calJ_P\models E \subsum C, \calJ_Q\models E \subsum D\}$.
			\end{itemize}
		\end{definition}
	\end{comment}
	
	\begin{definition}\label{def:sim-axioms}
		Let $\calK$ be a KB, $C,D$ be concepts, and $x,y$ be individual elements. 
		Then, for two subsets $\calJ_1,\calJ_2\subseteq \calK$, and $(X,Y)\in\{(x,y),(C,D)\}$, $\calJ_1\contrastdiff{X}{Y} \calJ_2$ denotes the \emph{set-difference between $\calJ_1$ and $\calJ_2$ relative to $(X,Y)$}, defined as follows:  
		\begin{itemize}
			\item $\calJ_1\contrastdiff{X}{Y} \calJ_2\dfn \{\alpha\in \calJ_1\mid  \alpha\not\in \calJ_2 \text{ or } \alpha[x\mapsto y] \not\in \calJ_2\}$, if $(X,Y)= (x, y)$.  
			\item $\calJ_1\contrastdiff{X}{Y} \calJ_2\dfn \{\alpha\in \calJ_1\mid  \alpha\not\in \calJ_2 \text{ or } \alpha[C\mapsto D] \not\in \calJ_2\}$, if $(X, Y)= (C,D)$.
			%Finally, we define $\calJ_1\contrastdiff{X}{Y} \calJ_2\dfn\emptyset$ in the remaining cases (if $X$ and $Y$ has different types).
		\end{itemize}
	\end{definition}
	\begin{comment}
		\begin{example}[Continue]\label{ex:contrast-issue2}
			Consider $\calJ_P= \{A_1(x),A_2(x), A_1\sqcap A_2 \subsum C \}$ as a justification for $P$ in $\calK$, and
			$ \calJ_{Q}= \{A_1(y), A_2(y), A_1\sqcap A_2 \subsum C \}$ a justification for $Q$ in $\calK$.
			Then, $\calJ_P\setminus \calJ_Q = \{ A_1(x),A_2(x)\}$ and $\calJ_Q\setminus \calJ_P = \{ A_1(y),A_2(y)\}$.
			Notice that $A_1(x), A_1(y) \in \calK$. Therefore, the two axioms do not differentiate the entailment of $P$ and the failure of $Q$ in $\calK$.
			Intuitively, to explain why $C(x)$ is true instead of $C(y)$, it suffices to note that $A_2(x)$ is true in $\calK$  whereas $A_2(y)$ is not.
			Put differently, although $A_1(x)$ is necessary to infer $P$, a similar axiom ($A_1(y)$) for $Q$ is also true in $\calK$.
			Therefore, $A_1(x)$ does not explain a difference between $P$ and $Q$ in $\calK$.
		\end{example}
	\end{comment}
	
	The set difference operator $(\contrastdiff{X}{Y})$ removes axioms that are \emph{similar} under renaming of concepts (resp., entities) to differentiate facts and foils.
	This intuition corresponds to including axioms required to explain $P$, such that their \textbf{corresponding} axioms for $Q$ are missing in the \textbf{history} of $Q$.
	%Let $(P,Q)$ be a CQ with the fact $P=C(x)$ and foil $Q\in \{C(y),D(x)\}$.
	%Then, by $\contrastdiff{P}{Q}$ we denote the set difference relative to $(P,Q)$ defined as: (1) $\contrastdiff{P}{Q}= \contrastdiff{x}{y}$, if $Q=C(y)$ and (2) $\contrastdiff{P}{Q}= \contrastdiff{C}{D}$ if $Q=D(x)$.
	For convenience, we denote the modified set difference as ``$\tiltedDagger$''.
	We now present our final proposal. % for  CEs in DLs that satisfies all the mentioned principles (P$1$--P$3$).
	As before, we find it convenient to denote individual components of a CE as axioms relevant for either fact or foil.
	
	\begin{definition}\label{def:cexp-modified} %(Contrastive Explanations with modified set difference.)
		Let $\calK$ be a KB and $(P,Q)$ be a CQ with a fact $P=C(x)$ and foil $Q\in \{C(y),D(x)\}$. 
		Moreover, let $\calJ(P)\in \JustAll(P,\calK)$, and $\calJ(Q)\in\JustAll(Q,\calK)$ with true assertions $\calJ_t(Q)\dfn \calJ(Q)\cap\calK$ and hypothesis $\calH(Q)=\calJ(Q)\setminus \calK$.
		A \emph{CE} for $(P,Q)$ in $\calK$ based on \emph{partition with set difference relative to $(P,Q)$} is a tuple 
		$\calE^\difftuple\dfn \langle \calE_\fact,\calE_\truefoil,\calE_\missingfoil \rangle$, where  
		\begin{itemize}
			\item $\calE_\fact \dfn \calJ(P){\tiltedDagger} \calJ_k(Q)$ \qquad (axioms only relevant for $P$),
			\item $\calE_\truefoil \dfn  \calJ_t(Q){\tiltedDagger}\calJ(P)$ \qquad  (true axioms only relevant for $Q$),
			\item $\calE_\missingfoil \dfn \calH(Q)$ \qquad \qquad \quad (missing axioms in the history of $Q$).
		\end{itemize}
		%	 $\calE^\difftuple\dfn \langle \calJ(P){\contrastdiff{P}{Q}} \calJ_k(Q), \calJ_k(Q){\contrastdiff{P}{Q}}\calJ(P), \calH(Q)\rangle$. %, where $\calJ(P)$ and $\calJ(Q)$ are justifications for $P$ and $Q$ respectively and $\calH(Q)$ is an abductive hypothesis for $Q$.
		
	\end{definition}
	We interpret the CE $\calE^\difftuple$ as: the axioms in $\calJ(P){\tiltedDagger} \calJ_t(Q)$ are \textbf{relevant} only for $P$ that do not find \textbf{a corresponding axiom} in the \textbf{history} of $Q$, and axioms in $\calJ_t(Q)\tiltedDagger \calJ(P)$ are relevant only for $Q$. 
	Furthermore, $\calH(Q)$ includes \textbf{missing} axioms responsible for $Q$'s \textbf{failure}.

	\begin{example}\label{ex:mod-part}
		Two CEs for $(P,Q)$ in $\calKrun$ are $\calE_1^\difftuple=\langle \calE_\fact,\emptyset, \calE_\missingfoil\rangle$ and $\calE_2^\difftuple=\langle \calE'_\fact,\calE'_\truefoil,\calE'_\missingfoil\rangle$, where,
		\begin{itemize}
			\item $\calE_\fact = \{\concept{publishesIn} (\text{Alice}, \text{ML})\}$,
			\item $\calE_\missingfoil =\{\concept{publishesIn} (\text{Bob}, \text{ML})\}$, and
			
			\vspace{.25em}
			\item $\calE'_\fact = \{ \concept{publishesIn} (\text{Alice}, \text{ML}), \exists\concept{pulishesIn}.\concept{AI}\subsum \concept{AI\text-Expert}\}$,
			\item $\calE'_\truefoil =\{ \exists\concept{worksIn}.\concept{AI}\subsum \concept{AI\text-Expert}\}$,
			\item $\calE'_\missingfoil =\{\concept{worksIn} (\text{Bob}, \text{ML})\}$. \qedexample
		\end{itemize}
	\end{example}

	%We conclude this section by noting that Principle~\ref{prop:similar} (P3) can also be implemented using similarity/difference between DL concepts.
	%For example, one can implement P3 to treat two assertions, $\concept{hasProfession}(A,y)$ and $\concept{hasOccupation}(B,y)$, for two individuals A and B, as similar. 
	%Our definitions can model this type of similarity with minor adjustments. 
	%This requires incorporating the similarities/differences between DL concepts and appropriately updating Definitions~(\ref{def:renamed-axioms}--\ref{def:cexp-modified}). 

	\subsection{Handling Inconsistencies in Explanations}\label{sec:conf}
	
	Our definitions for CEs in the previous section cannot handle the case when a justification for the foil does not exist.
	Such a situation can arise when some axioms in the KB conflict with the foil itself.
	This is inline with the setting of abduction~\cite{klarman2011abox,PukancovaH18,Del-PintoS19,Koopmann21a}, which require a hypothesis to be consistent with the KB.
	However, it might be impossible to find an explanation in certain cases involving incompatible facts and foils (i.e., both can not be entailed at the same time).
	
	To illustrate this, let us reconsider our KB $\calKrun$ from Example~\ref{ex:running-prelim} and add an axiom saying that ``Theorist'' and ``AI-Expert'' are two disjoint classes.
	Thus, we obtain $\calKrun'\dfn \calKrun\cup \{\concept{Theorist}\sqcap \concept{AI\text-Expert}\subsum \bot\}$.
	We have that $\calKrun'\models \concept{Theorist}(\text{Bob})$ due to the axioms $\{\concept{Publishes} (\text{Bob},\text{Math}),\concept{Theory}(\text{Math})\}$ in $\calKrun$.
	Now, let us reconsider our CQ $R\dfn (\concept{Hired}(\text{Alice})$, $ \concept{Hired}(\text{Bob}))$ from before.
	Here, $R$ does not admit any explanation apart from the trivial CE $\tup{\{\concept{Hired}(\text{Alice})\}, \emptyset, \{\concept{Hired}(\text{Bob})\}}$.
	However, one might still want to obtain a non-trivial yet interesting explanation, for instance: ``If Bob published in ML, he would have been hired, but publishing in ML makes him an AI-Expert which contradicts to Bob being a Theorist''.
	To mitigate this, we add a final component to our definition of CEs, namely the set of (possibly empty) conflicts incurred when trying to entail the given foil.
	In our running example, such a set of conflicts would include assertions $\{\concept{Publishes} (\text{Bob},\text{Math}),\concept{Theory}(\text{Math})\}$ in $\calK'$.
	
	For a CE $\calE\dfn \tup{\calE_\fact, \calE_\truefoil, \calE_\missingfoil}$ such that $(\calT, \calA\cup\calE_\missingfoil)\models\bot$, we define its \emph{conflict} set in $\calK$ to be the subset-minimal $\calC\subseteq \calA$ such that $(\calT, \calA' \cup \calE_\missingfoil)\not\models \bot$ where $\calA'\dfn \calA\setminus \calC$. %and $\calC' \cup \calE\not\models \bot$ for any proper subset $\calC'\subset \calC$.
	Intuitively, adding the missing axioms for the foil ($\calE_\missingfoil$) creates conflicts ($\calC$) with the assertions in $\calK$ and removing $\calC$ gives an alternative and consistent scenario in which the foil can still be entailed.
	Here, we consider $\calC\subseteq \calA$ to only include ABox assertions about foil that are conflicting with $\calE$. 
	This aligns with the existing notion of \emph{counterfactual scenario}~\cite{eiter2023contrastive,Koopmann26}.
	%This corresponds to the intuition that one is interested in assertions about the foil
	Henceforth, we assume that a CE $\calE$ has four components $\tup{\calE_\fact, \calE_\truefoil, \calE_\missingfoil, \calC}$ and we ignore $\calC$ in case there are no conflicts (i.e.,  $\calC=\emptyset$).
	
	\subsection{Preferred Contrastive Explanations}\label{sec:pref}
	We now focus on selecting preferred explanations when multiple choices are available.
	Traditionally, one expects explanations to be small and hence subset- or cardinality-minimal explanations are usually preferred. % as justifications and abductive explanations. 
	%Moreover, Horridge et al.~\cite{horridge2013toward} presented an approach to determining the cognitive complexity of justifications for entailments. 
	%Their primary motivation stems from the observation that users often find justifications difficult or impossible to understand. 
	%A similar observation arises with CEs when users attempt to comprehend the difference between two seemingly related events.
	%
	%\paragraph{Minimal Contrastive Explanations.}
	Thus, one may attempt to consider the similar criteria when defining preferred CEs.
	We argue that subset minimality is inadequate when preferred CEs are requested.
	Likewise, one may have additional preferences when faced with multiple cardinality-minimal CEs. 
	We illustrate these features next. %in the following.
	\begin{example}\label{ex:card-min}
		Let $\calK =(\calT,\calA)$ where $\calA =\{A_1(x), A_2(x)\}$ and $\calT = \{ A_1\sqcap A_2 \subsum E, E\subsum C\}$.
		Further, let $P\dfn C(x)$ be a fact and $Q\dfn C(y)$ be the foil.
		Then, $\calJ(P) =\{A_1(x), A_2(x),  A_1\sqcap A_2 \subsum E, E \subsum C\}$.
		The CQ $(P, Q)$ admits two CEs in $\calK$ of size $4$. 
		Namely, 
		\begin{itemize}
			\item $\calE^\difftuple_1 = \langle \{A_1(x),A_2(x), A_1\sqcap A_2\subsum E\},\emptyset,  \{E(y)\}\rangle$
			\item $\calE^\difftuple_2 =\langle\{A_1(x),A_2(x)\},\emptyset, \{A_1(y), A_2(y)\}\rangle$.
		\end{itemize}
		Both CEs are cardinality- and subset-minimal.
		But, $\calE^\difftuple_2$ 
		explains $Q$ \emph{in the context of $P$}, i.e., using the same concepts as $P$. \qedexample
	\end{example}

	Since CEs are derived from justifications for $P$ and $Q$, a \text{logic}al approach for preferred CEs involves evaluating the \emph{similarity} between the two justifications.
	This aligns with our intuition that CEs explain $P$ \emph{within the context} of $Q$, and highlight the shared characteristics between the fact and the foil in a KB.
	Alternatively, when explanations introduce conflicts with assertions already known about the foil, it might be desirable to prefer CEs that minimize such conflicts. 
	%This motivates a second criterion for defining best CEs. 
	Intuitively, one is interested in explanations whose missing assertions for the foil do not (or minimally) contradict the assertions already in place.
	However, computing conflict-minimal explanations can be computationally expensive~\cite{Koopmann21a,Koopmann26}. Therefore, we leave a detailed discussion and theoretical analysis of conflict-minimal CEs to future work. 
	In what follows, we compare two CEs in terms of axioms relevant for the fact and foil while the incurred conflicts do not contribute in the comparison.
	This way, the conflicts can still be computed efficiently since one does not require to find CEs with \emph{globally minimal} conflicts.
	
	In the following, we consider CEs according to Definition~\ref{def:cexp-modified}.
	Let $\calK$ be a KB, $(P,Q)$ be a CQ, and $\calE =\langle \calE_\fact, \calE_\truefoil, \calE_\missingfoil,\calC\rangle$ be a CE for $(P,Q)$.
	The size of $\calE$ ($|\calE|$) is defined as the number of axioms in $\calE_\fact\cup \calE_\truefoil\cup \calE_\missingfoil$.
	Moreover, let $\calE$ and $\calE'$ be two CEs with $ \calE'=\langle \calE'_\fact, \calE'_\truefoil, \calE'_\missingfoil\rangle $. 
	Then, we define $\calE\subseteq \calE'$ in a component-wise manner, i.e, $\calE_*\subseteq \calE'_*$ for each $*\in \{\fact, \truefoil,\missingfoil\}$.
	We define a CE $\calE$ for $(P,Q)$ to be subset (resp., cardinality) minimal in $\calK$ if there is no CE $\calE'\subsetneq \calE$ ($|\calE'|<|\calE|$) for $(P,Q)$ in $\calK$.
	
	%Similar to the case of the abductive hypothesis, subset minimality is inadequate when preferred CEs are requested (see examples in the Appendix).
	\begin{comment}
		Likewise, one may have additional preferences when faced with multiple cardinality-minimal CEs. 
		The two features are illustrated in the following example.
		\begin{example}\label{ex:card-min}
			Let $\calK =(\calT,\calA)$ where $\calA =\{A_1(x), A_2(x)\}$ and $\calT = \{ A_1\sqcap A_2 \subsum E, E\subsum C\}$.
			Further, let $P\dfn C(x)$ be a fact and $Q\dfn C(y)$ be the foil.
			Then, $\calJ(P) =\{A_1(x), A_2(x),  A_1\sqcap A_2 \subsum E, E \subsum C\}$.
			The CQ $(P, Q)$ admits two CEs in $\calK$ of size $4$. 
			Namely, $\calE^\difftuple_1 = \langle \{A_1(x),A_2(x), A_1\sqcap A_2\subsum E\},\emptyset,  \{E(y)\}\rangle$ and $\calE^\difftuple_2 =\langle\{A_1(x),A_2(x)\},\emptyset, \{A_1(y), A_2(y)\}\rangle$.
			Both CEs are cardinality- and subset-minimal.
			But, $\calE^\difftuple_2$ 
			explains $Q$ in the context of (using the same concepts as) $P$. \qedexample
		\end{example}
	\end{comment}
	We propose cardinality-based criteria for preferred CEs.
	Intuitively, those explanations are preferred, in which justifications for facts and foils \emph{resemble} the most or equivalently, when their \emph{syntactic divergence} is minimal. % (such as $\calE^\difftuple_2$ in Ex.~\ref{ex:card-min}).
	%an abductive hypothesis for the foil \emph{resembles} the justification 
	%We formalize this intuition in the following discussion.
	To formalize this, we define the symmetric difference ($\blacklozenge$) based on our set difference ($\tiltedDagger$)
	as  $A \blacklozenge B \dfn (A\tiltedDagger B) \cup (B\tiltedDagger A)$. We denote by $\calE_\foil\dfn \calE_\truefoil\cup\calE_\missingfoil$ all foil axioms.
	
	\begin{definition}\label{def:sim-index}
		Let $R$ be a CQ in $\calK$, and $\calE \dfn \langle \calE_\fact, \calE_\truefoil, \calE_\missingfoil,\calC\rangle$ be a CE for $R$ in $\calK$. 
		The $R$-divergence for $\calE$ is defined as $\sigma^\text{div}(\calE)\dfn |\calE_\fact {\blacklozenge}\calE_\foil|$.
	\end{definition}
	
	The intuition lies in preferring a CE in which the axioms required for the foil \emph{coincide} with those used to justify the fact. This preference is motivated by the observation that a skeptical user either understands the reasoning underlying the fact or seeks to compare the foil against it. Consequently, a better CE aligns the assertions relevant to the foil as closely as possible with those used for the fact.
	We define $\sigma^\text{div}$-preferred explanations as those that have the lowest divergence; hence, $\calE$ is $\sigma^\text{div}$-preferred if there is no explanation $\calE'$ with $\sigma^\text{div}(\calE')< \sigma^\text{div}(\calE)$.
	%Intuitively, $\sigma^\text{div}$ measures how \emph{similar} the reasons for $P$ and $Q$ are, in a CE for $(P,Q)$. %\vspace{-.1cm}
	
	\begin{example}\label{ex:sim-index}
		Consider the CEs $\calE_i^\difftuple$ for $i=1,2$ from Example~\ref{ex:mod-part}.  
		%	$=\langle C_{i1},C_{i2}, C_{i3}\rangle$ for $(P,Q)$, where: $C_{11} = \{\concept{publishesIn} (\text{Alice}, \text{ML})\}$, $C_{12}=\emptyset$, 
		%	$C_{13} =\{\concept{publishesIn} (\text{Bob}, \text{ML})\}$, and $C_{21} = \{ \concept{publishesIn} (\text{Alice}, \text{ML}), \exists\concept{pulishesIn}.\concept{AI}\subsum \concept{AI\text-Expert}\}$,
		%	$C_{22} =\{ \exists\concept{worksIn}.\concept{AI}\subsum \concept{AI\text-Expert}\}$, and 
		%	$C_{23} =\{\concept{worksIn} (\text{Bob}, \text{ML})\}$.
		It can be seen that $\sigma^\text{div}(\calE_1^\difftuple)=0$ and $\sigma^\text{div}(\calE_2^\difftuple)=4$. 
		Therefore, $\calE_1^\difftuple$ is $\sigma^\text{div}$-preferred. \qedexample  
		%	Further, employing alternative abductive hypotheses from Example~\ref{ex:running-abd} results in CEs with greater cardinality and higher $\sigma^\text{div}$-index than $\calE_1^\difftuple$. \qedexample
	\end{example}
	
	It is worth noting that cardinality-minimal CEs are not always $\sigma^\text{div}$-preferred.
	%Similarly, an explanation may have a syntactic  divergence of zero, although it is not cardinality minimal. 
	The following example shows how different minimality notions interact and why subset- or cardinality-based preferences alone might lead to unintuitive CEs.
	%That is, there can exist contrasive explanations with smaller cardinality but higher $\sigma^\text{div}$-index and vise versa.
	\begin{example} 
		Let $\calK =(\calT,\calA)$ be a KB where $\calT=\{(A_1\sqcap A_2)\subsum E, E\subsum C, (B_1\sqcap \ldots \sqcap B_n)\subsum E\}$ and $\calA =\{A_1(x),A_2(x), \neg A_1(y)\}\cup \{B_i(x)\mid i\leq n\}$.
		Consider the CQ $(C(x), C(y))$ and the two CEs for this instance as:
		\begin{itemize}
			\item $\calE^\difftuple_1= \langle \{A_1(x), A_2(x), (A_1\sqcap A_2)\subsum E\}, \emptyset, \{E(y)\}\rangle $ 
			\item $\calE^\difftuple_2 = \langle \{B_i(x) \mid i\leq n\}, \emptyset, \{B_i(y) \mid i\leq n\}\rangle$.
		\end{itemize}
		Then, $\calE^\difftuple_1$ is a cardinality minimal explanation for $(C(x),C(y))$ in $\calK$ (for $n\geq 3$), whereas $\calE^\difftuple_2$ has a lower divergence for each $n\in \mathbb N$. Intuitively, it is easy to compare $C(x)$ and $C(y)$ in $\calE^\difftuple_2$ as it includes \textit{similar} ABox assertions for both. \qedexample
	\end{example}
	%These examples highlight that the minimality criterion chosen—whether based on subset, cardinality, or similarity—can substantially affect which explanations are deemed preferred.
	%In particular, $\sigma^\text{div}$-preference captures semantic closeness between the fact and foil, which may better align with human intuitions of contrastiveness than purely syntactic notions of minimality.

	\section{Properties of Contrastive Explanations}\label{sec:theory}
	Here, we address several computational and otherwise interesting properties for CEs in our framework.
	We first consider the case of $\EL$ for which abduction is trivial whereas contrastive explanations still turns out to be non-trivial.
	%\subsection{Contrastive Explanations in $\EL$}
	%We now explore the properties of CEs for $\EL$ KBs.
	
	The DL $\EL$ allows concepts following the grammar rule $C\ddfn A \mid C\sqcap C \mid \exists r.C $.
	In particular, no constructor in $\EL$ causes inconsistencies.
	Therefore, the abduction is trivialized as for any $\EL$ KB $\calK$ and assertion $Q=C(y)$: $\{C(y)\}$ is an abductive hypothesis and a justification for $Q$.
	Nevertheless, the question of which abductive explanations lead to preferred (cardinality-minimal and $\sigma^\text{div}$-preferred) CEs, still remains relevant.
	%
	%\subsection{Characteristics of CEs in $\EL$}
	
	We begin by proving that there always exists a CE with empty conflicts for a CQ $(P,Q)$ in a KB $\calK$ formulated in $\EL$. 
	Unless specified otherwise, all results in this section hold for both (entity and concept) contrasts.
	%The complete proof details can be found in the technical appendix.
	
	\begin{proposition}\label{prop:el-exist}
		Let $\calK$ be an $\EL$ KB and $R=(P,Q)$ be a CQ.
		There exists at least one contrastive explanation for $R$ in $\calK$.
	\end{proposition}
	
	\begin{proof}
		Since $P$ is a fact, we have $\calK\models P$. 
		Let $\calJ\subseteq\calK$ be a justification for $P$ in $\calK$.
		Now, assume $Q\dfn \alpha$ be the foil in $R$. 
		Since $\calK$ is an $\EL$ KB and $\alpha$ is an $\EL$ concept assertion, $\calK\cup \{\alpha\}$ is consistent.
		As a result, $\calE\dfn \langle \calJ, \emptyset, \{\alpha\}\rangle$ is a contrastive explanation for $R$ in $\calK$.
		Moreover, the size of $\calE$ is bounded from above by $|\calK|+1$.
	\end{proof}
	
	%The following corollary regarding the upper bound for size/pref-index, follows 
	%It follows from the proof of Proposition~\ref{prop:el-exist} that the size (syntactic divergence) of the cardinality-minimal ($\sigma^\text{div}$-preferred) contrastive explanation for $R$ in $\calK$ is bounded by $k+1$, where $k= |\calK|$.
	
	%\begin{corollary}\label{cor:el-size}
	%	Let $\calK$ be an $\EL$ explanation context and $R=(P,Q)$ be a CQ.
	%	Then, the size ($\sigma^\text{div}$-index) of the cardinality-minimal ($\sigma^\text{div}$-preferred) contrastive explanation for $R$ in $\calK$ is bounded by $k+1$, where $k= |\calK|$.
	%\end{corollary}
	
	%We assume basic familiarity with the complexity class $\NP$~\cite{}.
	%Consequently, %it follows that the problem of deciding whether there exists a contrastive explanation of a given size is also hard for $\EL$ KBs.
	%Next, we prove that the problem 
	To determine if a CQ admits a CE of the given cardinality (divergence) $n\in\mathbb N$ is $\NP$-complete for \textbf{concept contrasts}. % and explanation contexts in $\EL$. %(with proofs given as additional materials).
	\begin{theorem}\label{cor:NP-hard}
		Let $\calK$ be an $\EL$ KB, $R\dfn (P,Q)$ a CQ for concept foil $Q$, $n\in\mathbb N$.
		It is $\NP$-complete to decide whether $R$ admits a contrastive explanation in $\calK$ of size (divergence) at most $n$.
	\end{theorem}
	
	\begin{proof}
		%Recall that the problem of determining the existence of a justification for an entailment $A\subsum B$ from an $\EL$ TBox $\calT$ of a given cardinality $n\in \mathbb N$ is $\NP$-complete~\cite[Theorem~3]{BaaderPS07}.
		We first prove the membership and then turn towards the hardness.
		
		\noindent\textbf{Membership:}
		We present a procedure (called, A) that performs in \linebreak non-deterministic polynomial time and decides whether an explanation of size at most $n$ exists for $R$ in $\calK$. We let $|\calK|=k $. 
		A has the following three steps.
		\begin{enumerate}
			\item Non-deterministically guess three sets of axioms $S_1,S_2,S_3$, such that 
			\begin{enumerate} %[label=(\alph*)]
				
				\item[(a)] $S_1\subseteq \calK$ with $|S_1|\leq k$,
				%\item
				(b) $S_2\subseteq \calK$ with $|S_2|\leq k$, 
				\item[(c)] $S_3$ is a set of ABox axioms over the signature of $\calK$ with $|S_3|\leq n$.
			\end{enumerate}
			\item Determine whether
			\begin{enumerate} %[label=(\alph*)]
				\item[(a)] $S_1\models P$,
				%\item
				(b) $S_2\cup S_3 \models Q$,
				\item[(c)] $|\calE_\fact|+|\calE_\truefoil|+|\calE_\missingfoil|\leq n$, where  $\calE_\fact \dfn S_1{\tiltedDagger} S_2$, $\calE_\truefoil\dfn S_2{\tiltedDagger}S_1$, and $\calE_\missingfoil \dfn S_3$. 
			\end{enumerate}
			\item Accept if each question (a)--(c) in Step 2 is true. 
		\end{enumerate}
		To prove the correctness of A, we notice that every contrastive explanation can be written as a collection of three axioms sets.
		We prove the following claim. % establishes the correctness.
		\begin{claim}\label{claim:NP}
			The procedure A involving Steps (1)--(3) performs in polynomial time and  correctly decides the existence of a CE for $R$ in $\calK$ of size at most $n$.
		\end{claim}
		\begin{claimproof}
			Suppose A accepts and let $S_1,S_2,S_3$ are the (non-deterministic) choices made by A in Steps (1a)--(1c) on an accepting branch.
			We observe that if  $S\models P$ for some $S\subseteq \calK$ with $|S|\leq k$, then there exists a justification for $P$ in $\calK$ of size at most $k$ (where $k=|\calK|$).
			Therefore, although $S_1$ may not be a subset-minimal set of axioms with $S_1\models P$ (Step (2a)), it includes a justification $\calJ(P)$ for $P$.
			Moreover, if there is no abductive hypothesis for $Q$ in $\calK$ of size $n$ (at most), then there can not exist a contrastive explanation for $R=(P,Q)$ in $\calK$ of size bounded by $n$.
			This holds since an abductive hypothesis $\calH(Q)$ is included in an explanation for $R$.
			As a result, the same claim holds for a justification for $Q$ in $S_2\cup S_3$ as in the case of $P$.
			We take $\calJ(P)\subseteq S_1$ and $\calJ(Q) \subseteq S_2\cup S_3$ be justifications (hence, subset-minimal) for $P$ and $Q$ in $\calK$, respectively.
			Note that $\calJ(P)$ and $\calJ(Q)$ exist due to Step (2a) and (2b).
			Then, $\calE = \langle \calE_\fact,\calE_\truefoil, \calE_\missingfoil\rangle$ is a contrastive explanation for $R$, where $\calE_*$ is the same as defined in Step (2c).
			Finally, $|\calE|\leq n$ since A accepts and Step (2c) returns true.
			
			Conversely, suppose a contrastive explanation $\calE=\langle \calE_\fact,\calE_\truefoil, \calE_\missingfoil\rangle$ of size at most $n$ exists for $R$ in $\calK$.
			Then, there is at least one choice of sets for the algorithm in Step (1). 
			Namely a justification $\calJ(P)$ for $P$, an abductive hypothesis $\calH(Q)$, and the corresponding justification $\calJ(Q)$ for $Q$ in $\calK$, that yields the contrastive explanation $\calE$. 
			Moreover, we have that $|\calJ(P)|\leq n$ and $|\calJ(Q)|\leq n+k$ since $|\calE|\leq n$ and $|\calK|= k$.
			Then, this choice of sets results in an accepting branch of A.
			%   Then, for every triple $(S_1,S_2,S_3)$ of sets as specified in Step (1), at least one condition from Step (2a)--(2c) is violated.
			%   As a result, for every justification $\calJ(P)$ for $P$ and $\calJ(Q)$ with abductive hypothesis $\calH(Q)$ for $Q$ in $\calK$: the third condition in Step (2) is violated. 
			%    In other words, the sum of axioms in the union of $C_i$ for $i\leq 3$ exceeds $n$.
			%    Since $\calK$ is an $\EL$ KB, there exists at least one  contrastive explanation for $R$ due to Proposition~\ref{prop:el-exist}.
			%    Therefore, Step (2c) must return false for our algorithm to reject the input. 
			%    Since this holds for all the non-deterministic choices (and therefore all combinations of sets in Step 1), $R$ can not have a contrastive explanation of size at most $n$.
			This completes the correctness proof.
			
			Regarding the runtime, note that the non-deterministic guess in the first step of A requires polynomial time since the size of each guessed set is bounded by $n$ and $k$.
			Further, the entailment in the second step can be performed in polynomial time for $\EL$ KBs.
			Finally, the last step consists of merely size comparison. 
			Consequently, the claim regarding the runtime follows. \qedclaim
		\end{claimproof}
		We present slight modifications to A for the membership regarding $\sigma^\text{div}$.
		This includes guessing a set $S_3$ in Step (1c) of size $k+n$ instead, and checking whether $|S_1 {\blacklozenge} S_3| \leq n$ in Step (2c).
		Then the correctness is established in the same way as in Claim~\ref{claim:NP} by observing that only an abductive hypothesis $\calH(Q)$ for $Q$ of size at most $n+k$ can yield a contrastive explanation with $\sigma^\text{div}$-index bounded by $n$.
		
		This establishes the membership. We now turn towards proving $\NP$-hardness.
		
		\noindent\textbf{Hardness:} Recall that the problem of determining the existence of a justification for an entailment $A\subsum B$ from an $\EL$ TBox $\calT$ of a given cardinality $n\in \mathbb N$ is $\NP$-complete~\cite[Theorem~3]{BaaderPS07}.
		We reduce from this %of computing a justification for an entailment $A\subsum B$ from an $\EL$ TBox $\calT$ of a given cardinality $k\in \mathbb N$ 
		to our problem as follows.
		
		Let $\calT$ be an $\EL$ TBox, and $A,B$ be two concepts over the signature of $\calT$.
		We let $\calK=(\calT,\calA)$ where $\calA =\{A(x)\}$.
		Moreover, we let $P\dfn B(x)$ be the fact.
		Then, every justification for $\alpha\dfn (A\subsum B)$ in $\calT$ gives a justification for $P$ in $\calK$.
		To be precise, $\calJ\models (A\subsum B)$ iff $\calJ\cup\{A(x)\}\models B(x)$ for every $\calJ\subseteq \calT$.
		Further, a justification for $(A\subsum B)$ of size $n$ yields a justification for $B(x)$ of size $n+1$.
		We let $D$ be a fresh (dummy) concept not in the signature of $\calT$ and define $Q\dfn D(x)$.
		The concept $D$ needs also be added to the signature of $\calK$. This can be achieved by letting $\calK'= (\calT',\calA)$, where $\calT' =\calT \cup \{D\subsum \top\}$.
		As a result, there is only one (trivial) abductive explanations for $Q$, namely $\{D(x)\}$.
		Finally, we seek a contrastive explanation for $(P,Q)$ of size $n+2$ in $\calK'$.
		This prove the claim by noting that every justification $\calJ(\alpha)$ for $\alpha$ in $\calT$ of size at most $n$ 
		corresponds to a CE $\langle \calJ(P), \emptyset, \{D(x)\}\rangle $ for $(P,Q)$ in $\calK'$ of size ($\sigma^\text{div}$) $n+2$ at most.
		%in such a way that only those explanations are considered that share the same TBox axioms 
	\end{proof}
	%In the following, we prove that 
	We next observe that not all the ``hardness results'' transfer from their counterparts for justifications in $\EL$.
	In particular, it is known that there exist families of knowledge bases $\calK_n$ for $n\in \mathbb{N}$ and an axiom $P$, such that $P$ has exponentially many justifications in $\calK_n$.
	We establish that a fact can admit exponentially many justifications, but still only few CEs for the CQ, involving the same fact.
	This also presents a case where a query may admit significantly large and complex justifications, but a corresponding CP may admit smaller CEs.
	
	\begin{theorem}\label{thm:exp-just}
		There is a family $\left(\calK_n \mid n\in \mathbb N\right)$ of $\EL$ KBs, a fact $P$ and a foil $Q$ such that $P$ admits exponentially many cardinality minimal justifications in each $\calK_n$, whereas the CQ $(P,Q)$ has a unique cardinality-minimal ($\sigma^\text{div}$-preferred) contrastive explanation in $\calK_n$. 
	\end{theorem}

	\begin{proof}
		Following Baader et al.~\cite{BaaderPS07}, we let $\calT_n \dfn \{ A_{i-1} \subsum X_i\sqcap Y_i,  X_i\subsum A_i, Y_i\subsum A_i \mid 1\leq i\leq n\}$ be a collection of TBoxes.
		Moreover, let $\calK_n = (\calT_n,\calA)$  be the family of KBs where $\calA = \{A_0(x)\}$.
		Assume, $P\dfn A_n(x)$ and $Q\dfn A_n (y)$.
		$P$ has exponentially many justifications in $\calK$~\cite[Example~1]{BaaderPS07}. 
		Namely, axioms $\calA\cup \{A_{i-1}\subsum X_i\sqcap Y_i,\mid 1\leq i\leq n\}$ are contained in every justification for $P$.
		Further, every justification contains exactly one axiom from the set $\{X_i\subsum A_i, Y_i\subsum A_i \}$ for each $ i\leq n$.
		Since there are two possibilities for each $i\leq n$, this results in $2^n$ many distinct justifications for $P$ in $\calK$ of size $2n+1$.
		Now, consider the foil $Q$ and the CE-instance $(P,Q)$.
		Then, the smallest CE (in size) is the one in which justifications for $P$ and $Q$ share the same TBox axioms. 
		Consequently, $\langle \{A_0(x)\}, \emptyset, \{A_0(y)\} \rangle$ is the only cardinality minimal ($\sigma^\text{div}$-preferred) contrastive explanation for $(P,Q)$ in $\calK_n$.
	\end{proof}

	We conclude this section by addressing the complexity of computing CEs for %considering entity contrasts.
	CQs involving entity contrasts. 
	Here, it suffices to compute a justification for the fact, which also yields a $\sigma^\text{div}$-preferred CE for the given CQ. 
	Intuitively, given a justification $\calJ$ for the fact, we can compute the \emph{corresponding} CE by ``simply'' renaming ABox axioms in $\calJ$ and finding the assertions for $Q$ missing in the KB $\calK$.
	This algorithm applies to any DL considered in this paper. However, for DLs involving negations or some other form of inconsistencies (such as disjointness axioms), one additionally has to compute the conflicts sets involved.
	
	\begin{theorem}\label{thm:EL-entity}
		Let $\calK$ be an $\EL$ KB, and $R$ be a CQ involving a fact $P\dfn C(x)$ and foil $Q\dfn C(y)$.
		Then, for every $\calJ\in \JustAll(P,\calK)$, there is a contrastive explanation $\calE_\calJ$ for $R$ such that $\sigma^\text{div}(\calE_\calJ) =0$.
	\end{theorem}
	
	\begin{proof}
		We construct a CE $\calE_\calJ$ from $\calJ$ by renaming individuals in the ABox assertions.
		Let $\calJ_\calA \dfn \calJ\cap \calA$ (resp., $\calJ_\calT \dfn \calJ\cap \calT$) denote the ABox (TBox) axioms in $\calJ$.
		Then, we observe that $\calJ_Q\dfn \calJ_\calT \cup \calJ_\calA[x\mapsto y]$ is a justification for $C(y)$.
		This can be established easily in every model of $\calJ$ via a \textit{bisimulation}~\cite{hitzler2009foundations,baader2003description} that maps $x$ to $y$ and keeps all the remaining individuals unchanged.
		
		Now, we let $\calJ_\calA[x\mapsto y]\cap \calK =\calJ_t $ and $\calJ_\calA[x\mapsto y]\setminus \calK =\calJ_m$ denote true and missing axioms for $Q$ is $\calK$, respectively. 
		This yields a contrastive explanation: $\calE_\calJ \dfn \langle \calE_\fact, \calE_\truefoil, \calE_\missingfoil\rangle $ for $(P,Q)$ where
		$\calE_\fact = \calJ_\calA \tiltedDagger \calJ_t$,
		$\calE_\truefoil =\calJ_t\tiltedDagger \calJ_\calA $ and finally $\calE_\missingfoil=\calJ_m$.
		%		$\calE_\calJ$ is a contrastive explanation for $(P,Q)$ in $\calK$.
		Here, the two justifications for $P$ and $Q$ share all the axioms in the TBox $\calT$ and every ABox assertion in the justification for $Q$ is simply a renaming of those in $\calJ$.
		As a result, we have that $\sigma^\text{div}(\calE_\calJ) =0$ by definition.
	\end{proof}

	As a consequence of Theorem~\ref{thm:EL-entity}, a contrastive explanation for CQs in $\EL$ KBs with syntactic semantic divergence can be computed in polynomial time.
	
	\begin{corollary}\label{cor:EL-entity-Ptime}
		Let $\calK$ be an $\EL$ KB, and $R$ be a CQ involving entity contrasts. %fact $P\dfn C(a)$ and foil $Q\dfn C(b)$.
		Then, a CE $\calE$ for $R$ in $\calK$ with $\sigma^\text{div}(\calE)=0$ can be computed in polynomial time.
	\end{corollary}
	
	Theorem~\ref{thm:EL-entity} is already invalid for KBs allowing the bottom operator ($\bot$) as one additionally needs to compute possible conflicts in the KB.
	Nevertheless, Corollary~\ref{cor:EL-entity-Ptime} still applies to $\EL_\bot$ ($\EL$ + disjointness axioms). The reason here is that one can compute the conflicts for $\EL_\bot$ KBs in polynomial time.
	Interestingly, one can extend the correctness results in Theorem~\ref{thm:EL-entity} to $\ALC$ KBs, however, one needs additional time to compute conflicts.
	
	\begin{theorem}\label{thm:ALC-entity}
		Let $\calK$ be an $\ALC$ KB, and $R$ be a CQ involving a fact $P\dfn C(x)$ and foil $Q\dfn C(y)$.
		A CE $\calE$ for $R$ such that $\sigma^\text{div}(\calE_\calJ) =0$ can be computed in polynomial time with access to an oracle that decides entailment in $\ALC$.
	\end{theorem}
	
	\begin{proof}
		As in the proof of Theorem~\ref{thm:EL-entity}, we compute a justification $\calJ$ for $P$ in $\calK$.
		This can be achieved by recursively checking each axiom of $\calK$ in turn and seeing whether the entailment of $C(x)$ still holds.
		This check in each step can be performed via oracle calls.
		Then, $\calJ_Q\dfn \calJ[x\mapsto y]$ is a justification for $Q$.
		Now, additionally, one has to check whether $(\calT, \calA\cup\calJ_Q)\models\bot$ and compute the minimal set of assertions $\calC\subseteq \calA$ such that $(\calT, \calA'\cup\calJ_Q)\not\models \bot$ where $\calA'\dfn \calA\setminus\calC$.
		The set $\calC$ can be similarly computed as in~\cite{Koopmann26}, via recursively finding justifications $\calJ_\bot$ for the inconsistency and adding one axiom from $\calJ_\bot$ to $\calC$.
		%	In other words, either $(\calT, \calA\cup\calJ_Q)\not\models \bot$ and hence we let $\calC\dfn \emptyset$, or there exists a justification $\calJ_\bot$ for the inconsistency.
		Precisely, we initialize $\calC=\emptyset$. Then, until $(\calT, (\calA\setminus \calC)\cup \calJ_Q)\models \bot$, we compute a justification $\calJ_\bot\subseteq \calA$ for the inconsistency and add some axiom from $\calJ_\bot$ to $\calC$.
		The set $\calJ_\bot$ can be computed via oracle calls to decide consistency for $\calK$.
		Since $\calA$ is finite, we obtain the mentioned runtime bounds.
	\end{proof}
	
	We conclude by providing the counterpart of Thm.~\ref{cor:NP-hard} for $\ALC$.
	Here, we get $\EXP$-completeness due to the complexity of entailment checking in $\ALC$.
	
	\begin{theorem}\label{cor:EXP-hard}
		Let $\calK$ be an $\ALC$ KB, $R\dfn (P,Q)$ a CQ for concept foil $Q$, $n\in\mathbb N$.
		It is $\EXP$-complete to decide whether $R$ admits a contrastive explanation in $\calK$ of size ($\sigma^\text{div}$) at most $n$.
	\end{theorem}
	%\alert{may also apply to the version without index? since entailment needs to be checked. }
	\begin{proof}
		\textbf{For membership}, one can try all hypotheses for $Q$ over the signature of $\calK$ in the given time, and entailment for $\ALC$ is in $\EXP$.
		\textbf{For hardness}, we reduce from instance checking for $\ALC$ KBs.
		Given a KB $\calK$, a concept $C$, and individual $x$: $\calK\models C(x)$ iff $C(x)$ has a justification in $\calK$ of size at most $|\calK|$ iff the CQ $(C(x), D(x))$ has a CE in $\calK$ of size (divergence) at most $|\calK|+1$, where we add a fresh concept $D$ in the signature for $\calK$.
		The claim follows due to the $\EXP$ complexity of the instance checking for $\ALC$.
	\end{proof}

	\section{Experimental Evaluation for Entity Contrasts}\label{sec:exp}
	We have implemented a first prototype to compute contrastive explanations involving entity contrasts.
	In the following, we outline our evaluation in detail.
	
	\subsection{Method for Computing CEs}
	We implemented the approach described in the proof of Theorem~\ref{thm:EL-entity} to compute preferred CEs in $\EL_\bot$.
	Precisely, we compute a justification for the fact~\cite{Kalyanpur2007Finding} and then determine a justification for the foil by renaming fact and foil individuals.
	The final CE $\calE$ is determined by distinguishing the axioms for the foil already present in the KB and computing any conflicts involved. % and hence compute .
	
	The source code for our algorithm and details on the experimental setup can be found online\footnote{\url{https://github.com/dice-group/CEs_DLs}}.
	Our algorithm is implemented in {Java~11}, using the {OWL~API}~\cite{OWL-API} as the core framework for ontology loading, axiom access, and reasoning. 
	%The OWL~API provides uniform support for handling ABox and TBox axioms and serves as the backbone of the experimental pipeline.
	%
	Reasoning services are provided by standard OWL reasoners, including both the {HermiT}~\cite{HERMIT} and {ELK}~\cite{ELK} reasoners. %HermiT is used to ensure sound and complete reasoning for expressive OWL~DL ontologies (e.g., $\ALC$), while ELK is employed for ontologies within the OWL~2~EL profile ($\EL$) to improve scalability. 
	Unless stated otherwise, the experimental results reported in this paper are generated using ELK.
	Finally, justifications for the fact and for the inconsistencies are computed using a black-box explanation approach based on the \texttt{BlackBoxExplanation} framework~\cite{Kalyanpur2007Finding}. 
	%This mechanism is used both to derive justifications for entailments and to compute explanations when abductive hypotheses introduce inconsistencies.
	
	%\alert{todo Balram: add server details, Experiments were conducted on a server with...}
	%Experiments were conducted on a dedicated Linux server running Debian~11 (Bullseye)  
	%The system is equipped with an x86\_64 architecture and an NVIDIA GeForce RTX~3090 GPU with 24\,GB of GPU memory, using NVIDIA driver version~560.35.05 and CUDA~12.6.
	Experiments were conducted on a dedicated Debian 11 (Bullseye) Linux server equipped with an x86\_64 architecture and an NVIDIA GeForce RTX 3090 GPU with 24 GB of memory.
	%The experimental framework was executed on the server named \texttt{limbo}, configured as a compute server.
	%
	Our experiments are executed via a custom Java-based experiment runner that automatically selects CQs (involving query, facts, and foils), executes the individual functions to compute justifications, and records detailed per-run statistics in CSV log files.
	For post-processing and evaluation, we use the Python-based library \emph{pandas} within a Jupyter notebook environment.
	These scripts aggregate the CSV logs, filter invalid runs, and compute the summary statistics reported in the following section.

	\subsection{Benchmark KBs for Experiments}
	To evaluate the proposed contrastive explanation approach, we conducted experiments on a collection of real-world OWL KBs that are publicly available~\cite{ORE_2015_ZENODO,ORE_2015_PAPER}. 
	The KBs we consider for our experiments are of different shapes and sizes, as they correspond to the ones used for ABox reasoning.
	In total, we consider approximately 100 KBs as initial candidates for our experiments.
	Due to practical memory and runtime constraints, we apply a filtering step based on the size and exclude KBs containing more than 10,000 axioms. %as reasoning and explanation generation on such KBs exceeded available memory limits in our setting.
%	Moreover, we also exclude KBs that contain an empty ABox.
	This results in 41 KBs spanning a variety of domains and differing in size, structure, and ABox density, providing a diverse benchmark for evaluation.
	
	For each selected KB, concepts (queries), facts, and foils are constructed automatically. A query corresponds to an OWL class selected from the signature of the underlying KB, subject to constraints that ensure the class is neither overly general nor trivial. 
	In particular, we consider queries that must have at least one positive instance (fact) and at least one negative instance (foil). 
	Moreover, we discard queries with an excessively large fraction of positive instances.
	%
	%\alert{Todo Balram: Add a table of summary for ontologies. Include: Signature size, number of individuals, TBox size, ABox size. For each, add average value and range}
	%
	\begin{table}[t]
		\centering
		\small
		\begin{tabular}{lcc}
			\toprule
			\textbf{Measure} & \textbf{Average} & \textbf{Range} \\
			\midrule
			Signature size  & 1{,}130.84 & [25 11594] \\
			Classes count & 1{,}027.01 & [8 11594] \\
			Properties & 103.83 & [0 529] \\
			Individuals & 381.40 & [0 3307] \\
			TBox size ($|\calT|$) & 1{,}334.66 & [55 8572] \\
			ABox size ($|\calA|$) & 653.64 & [0 8234] \\
			%Total axioms ($|\mathsf{Ax}(\mathcal{O})|$) & 3{,}205.00 & [113 20166] \\
			\bottomrule
		\end{tabular}
		\caption{Summary statistics of the benchmark KBs used in our experiments. For each measure, we report the average value and the observed range [min max].}
		\label{tab:ontology-summary}
	\end{table}
	Given a query \(C\), a fact individual \(a\) is selected from the set of individuals such that the assertion \(C(a)\) is entailed by the KB $\calK$. 
	To avoid trivial explanations, we give preference to entailed facts, i.e., facts that are inferred from the KB $\calK$ rather than explicitly present in $\calK$. A foil \(b\) is selected from the remaining individuals such that \(C(b)\) is not entailed from $\calK$. We further apply additional filtering to ensure that the foil is semantically related to the fact, increasing the likelihood of meaningful contrastive explanations.
	For each KB, we perform five independent runs using different randomly sampled combinations of queries, facts, and foils.
	Each run is subject to fixed upper bounds on the number of candidate queries, as well as facts and foils per query. 
	%We also exclude KBs for which no suitable (query, fact, foil) combinations can be found, but we do record these KBs for the final evaluation.
	%
%	Our fully automated setup allows us to benchmark the proposed method across a large and heterogeneous set of knowledge bases while ensuring that our results remain uniform and reproducible.

	\subsection{Evaluation Results}
	%\alert{Todo Balram: Expand the table by adding the range for each value. Also add: conflict size}
	
	We evaluate the proposed contrastive explanation method using quantitative measures that capture both the \textit{sizes} of individual justifications and missing axioms for foil, as well as their computational cost. 
	Moreover, since our algorithm computes CEs in which TBox axioms are \textit{shared} between the fact and fol, the reported measures are computed over ABox axioms.
	
	For each successful run, we record the size of the ABox justification for the fact ($\calE_\fact$). 
	%denoted as \textit{factJustABoxSize}. 
	%Here, $|\calE_\fact|$ captures how much instance-level evidence is required to explain the fact.
	%
	To analyze the foil, 
	%reuse of existing knowledge, 
	we record the number of ABox axioms for the foil that are already present in the KB ($\calE_\truefoil$), %after replacing the fact individual with the foil individual, 
	%This is denoted by \textit{alreadyPresentABoxSize} 
	%$|\calE_\truefoil|$ reflects how much of the explanation for the foil can already be supported by the original KB.
	%
	%The size of the abductive hypothesis for the foil 
	as well as the 
	%, denoted as \textit{deltaABoxSize}, corresponds to the 
	the number of missing ABox axioms ($\calE_\missingfoil$) required to entail the foil. 
	%$|\calE_\missingfoil|$ directly captures the explanatory effort required to justify the foil.
	%
	%\todo[inline]{I replaced (i) \textit{factJustABoxSize} by
		%$|\calE_\fact|$, (ii) \textit{alreadyPresentABoxSize} by 
		%$|\calE_\truefoil|$, and (iii) \textit{deltaABoxSize} by $|\calE_\missingfoil|$ everywhere in the table. 
		%}
	%
	In cases where adding $\calE_\missingfoil$ leads to an inconsistency with the KB, we additionally compute the (size of the) conflicts ($\calC$) over the ABox, i.e., minimal sets of axioms responsible for the inconsistency. 
	This allows us to quantify when CEs cause clashes with existing instance-level knowledge.
	For each run, we also record the total execution time in milliseconds from the start of the procedure to the completion of consistency checking and the final explanation generation. 
	The runtime includes (i) justification computation for the fact, (ii) computing existing and missing axioms for the foil from its justification, (iii) consistency testing, and (iv) extracting conflict. % (when applicable).
	
	To summarize results across KBs, we aggregate per-run statistics and report descriptive measures, including the average and range for the size of each component and the runtime.
	In particular, averages are computed over all valid runs, excluding cases where no suitable (query, fact, foil) combination could be constructed.
	Finally, we also report the consistency rate, defined as the percentage of runs in which the KB remains consistent after adding the missing axioms for the foil ($\calE_\missingfoil$). This measure indicates how often $\calE_\missingfoil$ yields valid and non-conflicting CEs.
	Aggregated results are computed from \textit{per-run CSV logs} and summarized in a combined results table derived from the file \texttt{overall\_averages.csv}.
	
	\begin{table}[t]
		\centering
		\begin{tabular}{l c c}
			\hline
			\textbf{Metric} & \textbf{Average} & \textbf{Range} \\
			\hline
			Fact justification ($|\calE_\fact|$) 
			& 1.35 & [1, 4] \\
			Present foil axioms ($|\calE_\truefoil|$) 
			& 0.35 & [0, 2] \\
			Missing foil axioms ($|\calE_\missingfoil|$) 
			& 1.00 & [0, 2] \\
			Runtime (ms) 
			& 392.41 & [24, 3854]\\
			Common axioms ratio (\%) 
			& 13.50 & -- \\ %[0, 100] \\
			Conflict rate (\%) 
			& 4.00 & -- \\ %[0, 100] \\
			%Consistency rate (\%) & 96.00 & -- \\ %[0, 100] \\
			\hline
		\end{tabular}
		\caption{Results on the computation of CEs across all benchmark KBs.}
		\label{tab:overall-abduction-results}
	\end{table}

	\paragraph{Summary of Evaluation.}
	The evaluation shows that our current prototype performs rather efficiently to compute CEs.
	Moreover, the computed explanations tend to be small, spanning up to four ABox axioms. 
	This is in line with the existing observations on the computation of justifications~\cite{Kalyanpur2007Finding} and further explains how only a small number of assertions are absent for the foil.
	Nevertheless, in most of the cases, fact and foil share too few axioms in their history, as seen by the  \textit{Common Axioms ratio}. 
	Eventually, there are only a small number of cases when the computed explanations conflict with the KB (\textit{Conflict rate}) regarding assertions for the foil, and the number of axioms in such cases also remains small.

	\section{Concluding Remarks}\label{sec:conclusion}
	
	We introduce the concept of contrastive questions and their answers to DL KBs and highlights several challenges in their implementation. 
	Our contributions further include establishing principles to address these challenges and providing definitions for explanations that adhere to these principles.
	We confirm that Lipton's framework for contrastive explanations can be tailored to facts and foils formulated in DLs. 
	This is achieved by %reformulating Lipton's definition for CEs and 
	presenting three principles governing CEs in DLs.
	%As this is the starting point for defining contrastive explanations in DLs, we have dedicated substantial effort to extensively highlight and address the involved issues.
	Towards the end, we presented criteria for preferred contrastive explanations and provided further theoretical analysis for CEs.
	We also present a prototype that demonstrates the practical readiness of our framework for CEs.
	
	One motivation for CEs lies in their specificity to the contrast in a question %motivating their definition 
	since a contrast specified by an inquirer reflects their preference. %in the context in which they seek an explanation.
	%We explain with an example how this is achieved via CEs.
	For instance, in the explanation context $\calKrun$ from Example~\ref{ex:running-prelim}, let us add an assertion about another individual, ``Carl'', namely, $\{\concept{publishesIn}(\text{Carl}, \text{ML})\}$.
	Then, an explanation for CQ $(\concept{Hired}(\text{Alice}), \concept{Hired}(\text{Carl}))$ depicts that Alice is qualified whereas Carl is not.
	The comparison of this explanation with the one for $(\concept{Hired}(\text{Alice}), \concept{Hired}(\text{Bob}))$ highlights how contrastive questions model an inquirer's \textit{preference} and a \textit{context} requested by them in an explanation.
	
	To the best of our knowledge, the notion of CEs has not previously been studied in the context of DLs, except for the specific case of entity foils~\cite{Koopmann26}.
%	We believe introducing the idea to DLs has opened a new research direction. 
	Our work can be considered preliminary in this area, and several interesting questions remain yet to be explored.
	Interestingly, our definitions extend to any DL for which a notion of abductive hypothesis exists. This follows since our approach essentially provides a mechanism for combining the two justifications to explain the difference between a fact and a foil.
	Considering the case of DLs with cardinality restrictions, one way to entail the foil in a CE is to (i) add missing role assertions via $\calE_\missingfoil$ to satisfy $\geq$-constraints, and (ii) remove assertions via $\calC$ if some $\leq$-constraints are violated.
	Nevertheless, a CE may not always exist when the two steps cannot be carried out simultaneously.

	\noindent
	\textit{Future Work.}
	First, a detailed exposition of our approach for more expressive DLs and an implementation that can handle conflicts efficiently is on our list.
	We aim towards a practical algorithm and an implementation that scales at least for $\ALC$ with \textit{cardinality} constraints, and \textit{role hierarchy}.
	The main complexity in our implementation arises from computing a justification for the fact and conflicts, which can be costly for large datasets; existing tools such as modules~\cite{ijcai2023p374} could help improve runtime. 
	Moreover, computing contrastive explanations for the concept foils also seems an interesting yet challenging task.
	On the empirical side, a user study should be conducted to assess the hypothesis of whether such explanations are preferred over \emph{plain} justifications.
	This would also allow us to evaluate different preference criteria for the best CEs and compare them. %(Sec.~\ref{sec:pref}).
	%
	%\textbf{Secondly}, it would be interesting to explore the cases of incompatible facts and foils.
	%We say that $P$ and $Q$ are incompatible when they are together inconsistent in a given explanation context $\calK$.
	%In this work, we assumed that $P$ and $Q$ are \emph{compatible} and there exists an abductive explanation for $Q$.
	%In the case of incompatible contrasts (e.g., in $\ALC$), our approach boils down to the standard justification-based approach.
	%It would be interesting to explore the case of incompatible contrasts, via the abduction problem for inconsistent KBs~\cite{du2015towards}.
	%Essentially, one can define CEs by considering the source of incompatibility. %(axioms involved in an inconsistency).
	%This will allow defining CEs, even if no abductive explanation for the foil exists.
	%
	On the theory side, it remains to explore further the complexity of the decision and verification problems concerning a subset or cardinality minimal explanation, or a preferred CE. % under a given optimality criteria. % of a given cardinality (or divergence).
	%For $\ALC$, one can obtain the same hardness for concept foils  %ExpTime-hardness 
	%as the complexity of instance checking (is $C(x)$ true in a KB $\calK$?).
	Lastly, one could look at how recent advances on abduction for inconsistent KBs can be used to define CEs while allowing conflicts~\cite{HaakWhyNot25}.
	%several other similarity or preference measures can be defined for preferred CEs, based on the semantic similarity~\cite{claudia2005semantic,lehmann2012framework,racharak2018personalizing}.
%	Then, the question of user preference regarding those measures is also worth exploring via user surveys.
	
	In line with existing literature, addressing privacy aspects is outside the scope of this paper. 
	Nevertheless, %in practical deployments, 
	explanations can be generated only from selected subsets of the KB, for instance by restricting the signature~\cite{Koopmann21b}. 
	Alternatively, sensitive axioms can be filtered or abstracted before explanations are produced. Moreover, our approach relies on symbolic reasoning and complexity-aware algorithms rather than data-intensive learning, and hence requires comparatively modest computational and data resources, aligning with the goals of frugal AI.

	\subsubsection*{Acknowledgments}
	We thank all anonymous reviewers for their valuable feedback. Research was funded by the  German Research Foundation (DFG), grant TRR 318/3 2026 – 438445824, 
	the Ministry of Culture and Science of North Rhine-Westphalia (MKW NRW) within projects WHALE (LFN 1-04) funded under the Lamarr Fellow Network programme and
	project SAIL, grant NW21-059D,
	and by the German Federal Ministry of Research, Technology and Space (BMFTR) within the project KI-Akademie OWL under the grant no 16IS24057B.
	
	%\clearpage 
	
	%\todo[inline]{ 1. comparing them based on subset-minimality is not straightforward; --> this is not an issue when we use flat abduction. \\
		%Other optimality measures, based on abduction problem: such as connection minimal explanations, or allowing complex assertions in the abductive hypothesis. }
	
	%\subsubsection{Some Rough notes}
	%when there is inconsistency involved, we can also define CEs to include the part that causes inconsitency. As depicted in the xample below. To achieve this, one should define CEs by emphasizing/prioritizing concept names rather than abducing the missing assertions. Q: can we then find CEs directly from justifications?
	%General setting (for entity contrasts) given a justification, look for x-assertions which are missing for y, or negated for y. define this to be CEs.
	\begin{comment}
		\begin{example}[Entity Foils]
			Let $\calT= \{{A}\subsum{B}, {B}\subsum{C}, C\subsum{D}\}$ and $\calA=\{A(x)\}$. 
			Then, $A(x)$ is the only justification for $D(x)$. 
			Now, CeXp for why $D(x)$ instead of $D(y)$ should cite: \emph{because $A(x)$
				instead of $A(y)$}. 
			But when $\calA=\{A(x),\neg A(y)\}$, there will be no CeXp.
			The question is, can one still answer the CeXp-problem?
			In principle, one can imagine the response: \emph{because $A(x)$ and $A$ is $B$, $B$ is $C$ and $C$ is $D$ whereas $\neg A(y)$ implies either $B(y)$ or $C(y)$ should be true}. But it defies the definition of the \emph{difference condition}.
		\end{example}
	\end{comment}

	\bibliographystyle{splncs04}
	\bibliography{main}
\end{document}